\documentclass{article}

\usepackage{microtype}
\usepackage{graphicx}
\usepackage{subcaption}
\usepackage{booktabs} 

\usepackage{hyperref}

\usepackage[preprint]{icml2026}

\usepackage{amsmath}
\usepackage{amssymb}
\usepackage{mathtools}
\usepackage{amsthm}

\usepackage[capitalize,noabbrev]{cleveref}

\theoremstyle{plain}
\newtheorem{theorem}{Theorem}[section]

\newtheorem{lemma}[theorem]{Lemma}
\newtheorem{corollary}[theorem]{Corollary}
\theoremstyle{definition}
\newtheorem{definition}[theorem]{Definition}
\newtheorem{assumption}[theorem]{Assumption}
\theoremstyle{remark}

\usepackage[textsize=tiny]{todonotes}

\icmltitlerunning{Descend or Rewind? Stochastic Gradient Descent Unlearning}

\begin{document}

\twocolumn[
  \icmltitle{Descend or Rewind? Stochastic Gradient Descent Unlearning}



  \icmlsetsymbol{equal}{*}

  \begin{icmlauthorlist}
    \icmlauthor{Siqiao Mu}{yyy}
    \icmlauthor{Diego Klabjan}{xxx}
  \end{icmlauthorlist}

  \icmlaffiliation{yyy}{Department of Engineering Sciences and Applied Mathematics, Northwestern University, Evanston, IL, United States}
  \icmlaffiliation{xxx}{Department of Industrial Engineering and Management Sciences, Evanston, IL, United States}

  \icmlcorrespondingauthor{Siqiao Mu}{siqiaomu2026@u.northwestern.edu}

  \icmlkeywords{Machine Learning, ICML, machine unlearning, differential privacy, stochastic gradient descent, gradient descent, first-order methods}

  \vskip 0.3in
]



\printAffiliationsAndNotice{}  

\begin{abstract}
Machine unlearning algorithms aim to remove the impact of selected training data from a model without the computational expenses of retraining from scratch. Two such algorithms are ``Descent-to-Delete" (D2D) and ``Rewind-to-Delete" (R2D), full-batch gradient descent algorithms that are easy to implement and satisfy provable unlearning guarantees. In particular, the stochastic version of D2D is widely implemented as the ``finetuning" unlearning baseline, despite lacking theoretical backing on nonconvex functions. In this work, we prove $(\varepsilon, \delta)$ certified unlearning guarantees for  \textit{stochastic} R2D and D2D for strongly convex, convex, and nonconvex loss functions, by analyzing unlearning through the lens of disturbed or biased gradient systems, which may be contracting, semi-contracting, or expansive respectively. Our argument relies on optimally coupling the random behavior of the unlearning and retraining trajectories, resulting in a sensitivity bound that holds \textit{in expectation} that yields $(\varepsilon, \delta)$ unlearning. We determine that D2D can yield tighter guarantees for strongly convex functions, but R2D is more appropriate for convex and nonconvex functions. Finally, we compare the algorithms empirically, demonstrating the strengths and weaknesses of each approach.
\end{abstract}

\section{Introduction}

Machine unlearning algorithms aim to remove the influence of specific data from a trained model without retraining from scratch. First introduced in \cite{caoOGMUL}, machine unlearning has attracted significant attention in recent years, driven by heightened concerns over user privacy, data quality, and the energy expenditures of training massive deep learning models such as large language models (LLMs). Regulatory pressures also play a role: provisions in the European Union’s General Data Protection
Regulation (GDPR), the California Consumer Privacy Act (CCPA), and the Canadian Consumer
Privacy Protection Act (CPPA), protect a user's ``right to be forgotten" by requiring that individuals be able to request removal of their personal information, whether stored in databases or retained by models \cite{sekharirememberwhatyouwant}. Since it is impractical to retrain from scratch every time a user requests removal of their data, we desire cost-saving methods for ``forgetting" information from models after training.

To mathematically characterize the extent of unlearning, a rich line of work investigates $(\varepsilon, \delta)$-certified unlearning algorithms, which are theoretically guaranteed to return model weights that are probabilistically indistinguishable to the weights obtained from retraining from scratch \cite{Guo, ginartmakingAI}. This is typically achieved by bounding the distance between the weights after unlearning and after retraining, and  injecting appropriately calibrated Gaussian noise, a technique from differential privacy (DP) known as the Gaussian mechanism \cite{dwork2014algorithmic}. Such provable guarantees can be a powerful alternative to empirical metrics of unlearning, which can be unreliable or misleading \cite{tu2025amia, zhangmiabad}.

However, many existing certified unlearning algorithms have limited practicality for modern deep learning settings. The vast majority are either second-order methods that require access to the  Hessian (or Hessian vector products) \cite{zhang2024towards, qiao2024efficientgeneralizablecertifiedunlearning, sekharirememberwhatyouwant, suriyakumar2022algorithms, Guo} or first-order methods that require computing the full gradient \cite{neel21a, 
chien2024langevin, mu2025rewindtodeletecertifiedmachineunlearning}. Both settings are computationally intractable for large-scale models, which are typically trained with \textit{stochastic gradient descent} (SGD) algorithms. While two recent works are able to show certified unlearning with stochastic gradients, they either prove a weaker version of certified unlearning \cite{koloskova2025certified} or only apply to convex functions \cite{chien2024stochastic}. To reflect realistic practices, we desire SGD-based unlearning algorithms for \textit{nonconvex} functions.

With this goal in mind, we revisit two existing first-order certified unlearning algorithms: Descent-to-Delete (D2D) \cite{neel21a}, designed for (strongly) convex loss functions, and Rewind-to-Delete (R2D) \cite{mu2025rewindtodeletecertifiedmachineunlearning}, designed for nonconvex functions. Both algorithms perform learning via full-batch gradient descent (GD). Upon unlearning, both algorithms perform additional GD steps on the loss function of the retained dataset, but D2D ``descends" from the final trained model, whereas R2D ``rewinds" to an earlier saved checkpoint during training before performing the unlearning steps.  Notably, even though D2D is only theoretically supported for GD on  strongly convex functions, its stochastic version is the basis for the ``finetuning" unlearning baseline method, which is implemented in \textit{virtually all} unlearning papers despite the fact that it underperforms on deep neural networks. The authors of \cite{mu2025rewindtodeletecertifiedmachineunlearning} argue that rewinding, instead of descending, is a more appropriate baseline method for nonconvex settings, but the  analysis in \cite{mu2025rewindtodeletecertifiedmachineunlearning} does not apply to stochastic gradients. Therefore, whether the stochastic versions of these two methods can provably unlearn is an important open question.

In this work, we prove that SGD versions of D2D and R2D, denoted as SGD-D2D and SGD-R2D respectively, do achieve $(\varepsilon, \delta)$ certified unlearning, and we derive their privacy-utility-complexity tradeoffs under looser assumptions than the original works. We first examine SGD-R2D with projection, for which we can achieve clean guarantees with minimal assumptions on strongly convex, convex, and nonconvex functions. On an unbounded domain, we can also prove certified unlearning for SGD-R2D for strongly convex, convex, and nonconvex functions, assuming that the second moment of the noise is relatively bounded and the loss is finite at initialization. Finally, we prove $(\varepsilon, \delta)$ certified unlearning for SGD-D2D on strongly convex functions.

To analyze SGD-R2D, we leverage the contractive properties of gradient systems to carefully track the divergence between the training trajectory and the retraining trajectory. By characterizing SGD on the original loss function as \textit{biased} or \textit{disturbed} SGD on the loss function of the retained data samples, we can achieve noise bounds that reflect the contractive, semi-contractive, and expansive properties of strongly convex, convex, and nonconvex functions respectively. We can also treat SGD-D2D as biased SGD, leveraging  additional favorable properties of strongly convex functions that allow the bias to be ``folded into" the classic convergence analysis, as long as the proportion of unlearned data is small enough. This approach is completely different from the original proof in \cite{neel21a}, which considers a highly constrained setting where the function is both strongly convex and Lipschitz continuous.

Our approach also relies on a key coupling argument; to achieve a minimal bound on the distance between unlearning and retraining, we can control the coupling between the two randomized SGD trajectories, yielding a sensitivity bound that holds \textit{in expectation} with respect to their joint distribution. By combining a first or second moment bound with Markov's inequality, we obtain a tail bound that holds with probability $1 - \delta$, achieving $(\varepsilon, 2\delta)$ indistinguishability at a small cost to the dependence on $\delta$. This novel technique allows a more flexible analysis of SGD algorithms, for which deterministic sensitivity bounds, required by the classic Gaussian mechanism,  can be challenging to obtain.

Our work reveals deeper insights into the difference between ``rewinding" and ``descending." On strongly convex functions, D2D yields tighter probabilistic bounds than R2D. However, D2D is not guaranteed to be more efficient than retraining from scratch if the initial point is sufficiently close to the global minimum. In contrast, R2D is \textit{always} more efficient than retraining from scratch. In fact, for strongly convex functions and constant noise, the number of unlearning iterations $K$ is \textit{better than sublinear} in $T$, the number of training iterations; it converges to a constant for large enough $T$, implying a \textit{potentially infinite computational advantage} $T - K$ for R2D unlearning.

We establish privacy-utility-complexity tradeoffs for SGD-R2D and SGD-D2D, which are easy to implement and are ``black-box," in that they only require noise injection at the end and do not require special algorithmic procedures during training. In addition, we perform experiments demonstrating the strengths of each approach, and illustrating that R2D can indeed circumvent the downsides of D2D in nonconvex settings, such as stalling in a stationary point. Our contributions can be summarized as follows,
\begin{itemize}
    \item We prove $(\varepsilon, \delta)$ certified unlearning for SGD-R2D with and without projection on strongly convex, convex, and nonconvex loss functions.
    
    \item We prove $(\varepsilon, \delta)$ certified unlearning for SGD-D2D on strongly convex functions, using a novel proof approach that circumvents the limiting assumptions of the original work.

    \item We conduct experiments comparing SGD-R2D and SGD-D2D on real-world datasets, confirming that  rewinding is more appropriate for nonconvex settings.
    
\end{itemize}

Code is open-sourced at the anonymous GitHub repository \url{https://anonymous.4open.science/r/r2d2-3753/}.

\begin{table*}[th]
\label{table:firstordermethods}
\centering
\caption{Comparison of first-order algorithms for certified unlearning.}
\begin{small}
\begin{tabular}{llll}
\toprule
\textbf{Algorithm}                   & \textbf{Method}                          & \textbf{Noise}                                 & \textbf{Loss Function}                       \\ \midrule
D2D \cite{neel21a}                  & GD w/ regularization          & At the end  & Convex            \\
R2D \cite{mu2025rewindtodeletecertifiedmachineunlearning} & GD                              & At the end  & Nonconvex                           \\
Langevin Unlearning \cite{chien2024langevin} & Projected GD      & At every step                         &  Nonconvex \\
Langevin SGD \cite{chien2024stochastic}         & Projected SGD                   & At every step                         & Convex            \\
PABI \cite{koloskova2025certified}        & Clipped SGD w/ regularization & At every step          & Nonconvex                           \\ 
PSGD-R2D (Theorem \ref{thm:unlearningprojected})         & Projected SGD                             & At the end  &  Nonconvex                   \\ 
SGD-R2D (Theorem \ref{thm:unlearningr2d})         & SGD                             & At the end  &  Nonconvex                   \\ 
SGD-D2D (Theorem \ref{thm:unlearningstronglyconvex})         & SGD                             & At the end  & Strongly convex                     \\\bottomrule
\end{tabular}
\end{small}
\vskip -0.2in
\end{table*}

\section{Related Work}
\label{section:priorwork}

\textbf{Certified unlearning.} The term machine unlearning was first coined in \cite{caoOGMUL} to address deterministic data deletion. The works \cite{ginartmakingAI, Guo} introduce a probabilistic notion of unlearning that uses differential privacy to theoretically certify the level of deletion. While there exists many second-order certified unlearning algorithms \cite{zhang2024towards, qiao2024efficientgeneralizablecertifiedunlearning, sekharirememberwhatyouwant, suriyakumar2022algorithms, basaran2025a} that require computation of the Hessian or Hessian vector-products, this work focuses on first-order methods (Table 1), which only require access to the gradient and are far more tractable for large-scale problems in terms of computation and storage. Existing first-order methods, detailed in Table 1, include full gradient methods as well as two SGD methods which we now discuss in detail.

First, \cite{chien2024stochastic}  analyzes noisy projected stochastic gradient descent and obtains theoretical guarantees on strongly convex and convex functions, by leveraging the uniqueness of the limiting distribution of the training process. However, their algorithm does not apply to nonconvex functions. Second, \cite{koloskova2025certified} proposes performing SGD with clipping, regularization, and noise on the retain set during unlearning, which advantageously does not require the function to be Lipschitz smooth or convex. However, their algorithm only satisfies a \textit{weaker} ``post-processing" definition of unlearning, which ensures that the model obtained from unlearning one subset is indistinguishable from the model obtained from unlearning another \cite{allouah2025the, sekharirememberwhatyouwant, basaran2025a}. In fact, since the analysis in \cite{koloskova2025certified} relies on privacy amplification by iteration (PABI), it likely cannot be extended to obtain the stronger result in this paper, which achieves indistinguishability between the unlearned and retrained models. 
Both \cite{chien2024stochastic} and \cite{koloskova2025certified} require noise injection at every iteration. In contrast, the D2D and R2D frameworks are ``black-box" in that they only require  noise once after training and once after unlearning. Therefore, they can be applied to models pretrained without any special procedures.

Certified unlearning algorithms have also been developed for specific settings, such as linear and logistic models \cite{Guo, izzo21a}, graph neural networks \cite{chien2022certifiedgraphunlearning}, minimax problems \cite{liu2023certified},  federated learning \cite{pmlr-v238-fraboni24a}, adaptive unlearning requests \cite{gupta2021adaptive, chourasia2023forget}, and online learning \cite{hu2025onlinelearningunlearning}. For additional literature review, see Appendix \ref{section:morerelatedwork}.

\section{Algorithm}

Let $\mathcal{D} = \{z_1,..., z_n \}$ be a training dataset of $n$ data points drawn  from the data distribution $\mathcal{Z}$. Let $A : \mathcal{Z}^n \to \mathbb{R}^d$ be a randomized learning algorithm that trains on $\mathcal{D}$ and outputs a model with weight parameters $\theta \in \mathbb{R}^d$. Typically, the goal of a learning algorithm is to minimize $\mathcal{L}_\mathcal{D}(\theta)$, the empirical loss on $\mathcal{D}$, defined as follows
\begin{equation}
\label{eq:finitesum}
    \mathcal{L}_{\mathcal{D}}(\theta) = \frac{1}{n} \sum_{i = 1}^n \ell(z_i ; \theta),
\end{equation}
where $\ell(z_i ; \theta)$ represents the loss on the sample $z_i$. A standard approach for minimizing $\mathcal{L}_{\mathcal{D}}(\theta)$ is gradient descent:
$$\theta_{t} = \theta_{t-1} - \eta \nabla \mathcal{L}_{\mathcal{D}}(\theta_{t-1}).$$
However, computing the full gradient can be computationally intractable for large models and datasets, motivating SGD algorithms where at each iteration $t$ we uniformly sample $\mathcal{B}_t$, a mini-batch of size $b$ \textit{with replacement} from $\mathcal{D}$ and construct a random gradient estimator $g_{\mathcal{B}_t}(\theta)$ as follows:
$$g_{\mathcal{B}_t}(\theta) =  \frac{1}{b} \sum_{i \in \mathcal{B}_t} \nabla \ell(z_i ; \theta),$$
where $\mathbb{E}[g_{\mathcal{B}_t}(\theta)] = \nabla \mathcal{L}_{\mathcal{D}}(\theta)$. Vanilla SGD algorithms update $\theta$ with this gradient estimator:
\begin{equation}
    \label{eq:SGD}
    \theta_{t} = \theta_{t-1}- \eta g_{\mathcal{B}_t}(\theta_{t-1}).
\end{equation}
Finally, we project the iterates onto a nonempty and convex set $\mathcal{C} \subseteq \mathbb{R}^d$. When $\mathcal{C} = \mathbb{R}^{d}$, this is vanilla SGD, and when $\mathcal{C}$ is closed, this is projected SGD (PSGD). We update $\theta$ as follows,
\begin{align}
\label{eq:PSGD}
    \theta_{t} &= \Pi_{\mathcal{C}}(\theta_{t-1}- \eta g_{\mathcal{B}_t}(\theta_{t-1})),
\end{align}
where $\Pi_{\mathcal{C}}(x)$ denotes the (unique) projection of $\theta$ onto $\mathcal{C}$,
\begin{equation}
\label{eq:projection}
    \Pi_{\mathcal{C}}(x) = \arg \min_{x' \in \mathcal{C}} \lVert x - x' \rVert.
\end{equation}

Let $Z \subset \mathcal{D}$ denote a subset of size $m$ we would like to unlearn, which we call the unlearned set, and let $\mathcal{D}'  = \mathcal{D} \backslash Z$ denote the retained set.
We desire an unlearning algorithm $U$ that removes the influence of  $Z$ from the output of the learning algorithm $A(\mathcal{D})$, such that the model parameters obtained from $U(A(\mathcal{D}), \mathcal{D}, Z)$ are probabilistically \textit{indistinguishable} from the output of $A(\mathcal{D'})$. To formalize this requirement, we borrow from DP the concept of $(\varepsilon, \delta)$-indistinguishability, which requires that the marginal distributions of two algorithm outputs resemble one another. 

Throughout, for a fixed dataset $\mathcal{D}$ and unlearning set $Z$, all probabilities and expectations are taken with respect to the internal randomness of the learning and unlearning algorithms (mini-batch index sampling and additive Gaussian noise). We assume that these random choices are generated from a fixed distribution that does not depend on the data values in $\mathcal{D}$.

\begin{definition}

\cite{dwork2014algorithmic, neel21a} Let $X$ and $Y$ be outputs of randomized algorithms. We say $X$ and $Y$ are $(\varepsilon, \delta)$-indistinguishable if for all $S \subseteq \Omega$,
\begin{align*}
    \mathbb{P}[X \in S] \leq& e^{\varepsilon} \mathbb{P}[Y \in S] + \delta,\\
    \mathbb{P}[Y \in S] \leq& e^{\varepsilon} \mathbb{P}[X \in S] + \delta.
\end{align*}
    
\end{definition}
In the context of DP, $X$ and $Y$ are the learning algorithm outputs on neighboring datasets that differ in a single sample, $\varepsilon$ is the privacy loss or budget, and $\delta$ is the probability that the privacy guarantees might be violated. 

\begin{definition}
\label{def:unlearning}
Let $A$ be a randomized learning algorithm and $U$ a randomized unlearning algorithm. Then $U$ is an $(\varepsilon, \delta)$ certified unlearning algorithm for $A$ if for all $Z \subset \mathcal{D}$, $U(A(\mathcal{D}), \mathcal{D}, Z)$ and $A(\mathcal{D} \backslash Z )$ are  $(\varepsilon, \delta)$-indistinguishable with respect to the algorithmic randomization.
    
\end{definition}

Now we describe the R2D and D2D unlearning frameworks. For both, the learning algorithm involves $T$ training iterations on $\mathcal{L}_{\mathcal{D}}$, and the unlearning algorithm involves $K$ iterations on $\mathcal{L}_{\mathcal{D}'}$. The difference lies in the initialization of the unlearning algorithm: for R2D, the unlearning algorithm is initialized at the $T-K$th iterate of the learning trajectory, whereas for D2D, it is initialized at the $T$th or last iterate of the learning trajectory. Finally, Gaussian noise is added at the end to achieve $(\varepsilon, \delta)$-indistinguishability. These frameworks are easily combined with various gradient methods, including SGD and PSGD.

\begin{algorithm}[H]
\caption{SGD-R2D Learning Algorithm}\label{alg:learnr2d}
\begin{algorithmic}
\REQUIRE dataset $\mathcal{D}$, initial point $\theta_0$, domain $\mathcal{C}$ 

\FOR{t = 1, 2, ..., T}
\STATE Uniformly sample with replacement  $\mathcal{B}_t 
\sim \mathcal{D}$  
$\theta_t = \Pi_{\mathcal{C}}(\theta_{t-1} - \eta g_{\mathcal{B}_t}(\theta_{t-1}))$ 
\ENDFOR
\STATE Save checkpoint $\theta_{T-K}$
\STATE Use $\tilde{\theta} = \theta_T + \xi$, where $\xi \sim \mathcal{N}(0, \sigma^2)$, for inference
\STATE Upon unlearning, execute Algorithm \ref{alg:unlearnr2d}

\end{algorithmic}
\end{algorithm}

\begin{algorithm}[H]
\caption{SGD-R2D Unlearning Algorithm}\label{alg:unlearnr2d}
\begin{algorithmic}
\REQUIRE dataset $\mathcal{D'}$, model checkpoint $\theta_{T-K}$, domain $\mathcal{C}$
\STATE $\theta''_0 = \theta_{T-K}$
\FOR{t = 1, ..., K}
\STATE Uniformly sample with replacement  $\mathcal{B}'_t 
\sim \mathcal{D}'$  

\STATE $\theta''_t = \Pi_{\mathcal{C}}(\theta''_{t-1} - \eta g_{\mathcal{B}'_t}(\theta''_{t-1}))$ 
\ENDFOR
\STATE  Use $\tilde{\theta} = \theta''_K + \xi'$, where $\xi' \sim \mathcal{N}(0, \sigma^2)$, for inference

\end{algorithmic}
\end{algorithm}
\vspace{-0.1in}
\begin{algorithm}[H]
\caption{SGD-D2D Learning Algorithm}\label{alg:learn1}
\begin{algorithmic}
\REQUIRE dataset $\mathcal{D}$, initial point $\theta_0 $ 

\FOR{t = 1, 2, ..., T}
\STATE Uniformly sample with replacement  $\mathcal{B}_t \sim \mathcal{D}$
\STATE $\theta_t = \theta_{t-1} - \eta g_{\mathcal{B}_t}(\theta_{t-1})$ 
\ENDFOR
\STATE Use $\tilde{\theta} = \theta_T + \xi$, where $\xi \sim \mathcal{N}(0, \sigma^2)$, for inference
\STATE Upon unlearning, execute Algorithm \ref{alg:unlearn1}.

\end{algorithmic}
\end{algorithm}
\vspace{-0.1in}
\begin{algorithm}[H]
\caption{SGD-D2D Unlearning Algorithm}\label{alg:unlearn1}
\begin{algorithmic}
\REQUIRE dataset $\mathcal{D'}$, model checkpoint $\theta_{T}$
\STATE $\theta''_0 = \theta_{T}$
\FOR{t = 1, ..., K}
\STATE  Uniformly sample with replacement  $\mathcal{B}'_t \sim \mathcal{D}'$
\STATE $\theta''_t = \theta''_{t-1} - \eta g_{\mathcal{B}'_t}(\theta''_{t-1})$ 
\ENDFOR
\STATE  Use $\tilde{\theta} = \theta''_K + \xi'$, where $\xi' \sim \mathcal{N}(0, \sigma^2)$, for inference

\end{algorithmic}
\end{algorithm}

\section{Analyses}

We first discuss general assumptions required for this work. We assume that the loss function $\ell$ is differentiable and Lipschitz smooth (Assumption \ref{assump:lipschitz}), standard requirements for the analysis of SGD algorithms.

\begin{assumption}
\label{assump:lipschitz}
    For all $z \in \mathcal{Z}$, the function $\ell(z ; \theta)$ is differentiable in $\theta$ and Lipschitz smooth in $\theta$ with constant $L > 0$ such that for any $\theta_1, \theta_2 \in \mathbb{R}^d$,
$$\lVert \nabla \ell(z ; \theta_1) - \nabla \ell(z ; \theta_2) \rVert \leq L \lVert \theta_1 - \theta_2 \rVert.$$
\end{assumption}
\noindent In addition to nonconvex functions, we analyze the cases of convex and strongly convex functions, defined below in Assumptions \ref{assump:convex} and \ref{assump:stronglyconvex}.

\begin{assumption}
\label{assump:convex}
    For all $z \in \mathcal{Z}$, the function $\ell(z; \theta)$ is convex in $\theta$ such that for any $\theta_1, \theta_2 \in \mathbb{R}^d$,
    \begin{align}
        \ell(z; \theta_1) \geq \ell(z; \theta_2) + \langle \nabla \ell(z; \theta_2) , \theta_1 - \theta_2 \rangle. 
    \end{align}
\end{assumption}

\begin{assumption}
\label{assump:stronglyconvex}
For all $z \in \mathcal{Z}$, the function $\ell(z; \theta)$ is $\mu$-strongly convex such that for any $\theta_1, \theta_2 \in \mathbb{R}^d$,
    \begin{equation*}
        \ell(z; \theta_1) \geq \ell(z; \theta_2) + \nabla \ell(z; \theta_2)^T (\theta_1 - \theta_2) + \frac{\mu}{2} \lVert \theta_2 - \theta_1 \rVert^2.
    \end{equation*}
\vspace{-0.3in}
\end{assumption}
Now we broadly describe our proof approach. Our analysis relies on carefully tracking the SGD trajectories on $\mathcal{L}_{\mathcal{D}}$ and $\mathcal{L}_{\mathcal{D}'}$ and bounding the outputs of training on $\mathcal{L}_{\mathcal{D}'}$ and unlearning on $\mathcal{L}_{\mathcal{D}}$. Specifically, we have that $\{\theta_t\}_{t=0}^T$ and $\{\theta'_t\}_{t=0}^T$ represent the learning iterates on $\mathcal{L}_\mathcal{D}$ and $\mathcal{L}_\mathcal{D'}$ respectively, starting from $\theta'_0 = \theta_0$. In addition, $\{\theta''_t\}_{t=0}^{K}$ represents the unlearning iterates on $\mathcal{L}_{\mathcal{D}'}$, where for D2D we have $\theta''_0 = \theta_{T}$ and for R2D we have $\theta''_0 = \theta_{T-K}$. The goal is to bound the distance between $\theta'_T$ and $\theta''_K$. 

As highlighted in the introduction, our analysis relies on a key coupling argument leading to a sensitivity bound that holds \textit{in expectation}. We couple the randomization (the ``coin flips") of the learning and unlearning algorithms to minimize the distances between $\theta_t$ and $\theta'_t$, as well as $\theta'_t$ and $\theta''_t$, by choosing the batches sampled at each time step so that they coincide as much as possible. This approach is similar to existing proofs in differential privacy \cite{abadiDP}. However, these works evaluate the deviation on the full datasets $\mathcal{D}$ and $\mathcal{D}'$ and combine privacy amplification by subsampling with noise at every step to achieve DP. In contrast, to maintain the output perturbation structure (which is amenable to black-box unlearning), we evaluate the sensitivity over the coupling over all steps, obtaining a sensitivity bound on $\lVert \theta'_T - \theta''_K \rVert$ that holds in expectation with respect to the chosen joint distribution. This is combined with Markov's inequality to yield a tail bound of probability $1 - \delta$, producing an $(\varepsilon, 2 \delta)$ guarantee.

\begin{lemma}
\label{lemma:indistinguishfirstsecond}
    Let $x$ and $y$ be random variables over some domain $\Omega$, and let $\xi, \xi'$ be  draws from the Gaussian distribution $\mathcal{N}(0, \sigma^2)$. Then for $0 < \varepsilon \leq 1$ and $\delta > 0$, we have the following.
    \begin{enumerate}
        \item If the second moment of the distance between $x$ and $y$ is bounded as $
            \mathbb{E}[\lVert x - y \rVert^2] \leq \Sigma^2$, then
$X = x + \xi$, $Y =y + \xi'$ are $(\varepsilon, 2 \delta)$-indistinguishable if
$$\sigma = \frac{\Sigma}{\varepsilon} \sqrt{\frac{2 \log (1.25/\delta)}{\delta}}.$$
    \item If the first moment of the distance between $x$ and $y$ is bounded as $\mathbb{E}[\lVert x - y \rVert] \leq \Sigma$,
then $X = x + \xi$, $Y =y + \xi'$ are $(\varepsilon, 2 \delta)$-indistinguishable if
            $$\sigma = \frac{\Sigma \sqrt{2 \log (1.25/\delta)}}{\varepsilon \delta}.$$
    \end{enumerate}
\end{lemma}
\noindent\textit{Proof.} See Appendix \ref{sec:relaxedGM_append}.

Lemma \ref{lemma:indistinguishfirstsecond} maintains differential privacy because the algorithmic randomness (for SGD, the mini-batch indices sampled  uniformly from $\{1,..., n \}$) does not depend on the actual data values. Like classic DP methods, we require that the bound $\Sigma$ also does not encode  information about the specific dataset,  thereby preventing data leakage. We emphasize that our approach differs from $(\varepsilon, \delta, \gamma)$-\textit{random differential privacy} \cite{hallRDP, rubinstein17a}, which considers whether $(\varepsilon, \delta)$-privacy holds with probability $1 - \gamma$ \textit{with respect to the sampling of the dataset} from a well-behaved data distribution. 

Our analysis of the R2D framework relies on characterizing unlearning as \textit{biased} or \textit{disturbed} SGD on a gradient system. The SGD iterates on $\mathcal{L}_{\mathcal{D}}$ are biased SGD iterates on $\mathcal{L}_{\mathcal{D}'}$, where the \textit{unlearning bias} is the difference between the gradients on $\mathcal{L}_{\mathcal{D}}$ and $\mathcal{L}_{\mathcal{D}'}$.
Upon unlearning, we perform \textit{unbiased} SGD on $\mathcal{L}_{\mathcal{D}'}$. Our results on strongly convex, convex, and nonconvex functions align with nonlinear contraction theory \cite{kozachkovgeneralization, sontag2021remarks} that predicts that $\theta_t$ and $\theta'_t$ will stay close for perturbed contracting systems (strongly convex functions), diverge linearly on perturbed semi-contracting systems (convex functions), and diverge exponentially otherwise (nonconvex functions). Like in \cite{mu2025rewindtodeletecertifiedmachineunlearning}, rewinding reverses the accumulation of these disturbances, drawing the unlearning trajectory closer to the retraining trajectory. Moreover, rewinding allows us to completely eliminate the impact of noise during unlearning by coupling. However, because the contraction analysis only holds with respect to a Riemannian metric, we can only achieve a first-moment sensitivity bound.

In contrast, the analysis of D2D relies on additional ``tricks" available for  the analysis of biased gradient descent on strongly convex functions. By relying on the existence and attractivity of the global minimum of strongly convex functions, we can leverage the faster convergence of biased SGD algorithms to achieve a tighter second-moment bound that has a better dependence on $\delta$.

\subsection{PSGD-R2D: Rewinding with Projected SGD}
\label{sec:PSGD_main}
We first consider projected SGD algorithms, for which we can achieve clean unlearning guarantees with minimal assumptions. In particular, in this setting the gradient of the loss function is uniformly bounded over $\mathcal{C}$. This allows us to uniformly bound the unlearning bias without a bounded gradient assumption or clipping.

\begin{theorem} 
\label{thm:unlearningprojected}
Suppose that the loss function $\ell$ satisfies Assumptions \ref{assump:lipschitz}, and within the closed, bounded, and convex set $\mathcal{C} \subset \mathbb{R}^d$, the gradient of $\ell$ is uniformly bounded by some constant $G \geq 0$ such that for all $z \in \mathcal{Z}$ and $\theta \in \mathcal{C}$, $\lVert \nabla \ell(z; \theta) \rVert \leq G$.
We implement PSGD-R2D (Algorithms 1 and 2) with 
$\sigma = \frac{\Sigma \sqrt{2 \log (1.25/\delta)}}{\varepsilon\delta}.$ 
Then for $0 < \varepsilon \leq 1$ and $\delta > 0$, PSGD-R2D is an $(\varepsilon, 2\delta)$-unlearning algorithm for the following choices of $\Sigma$:
\begin{itemize}
    \item For general functions, we have 
    \begin{equation*}
        \Sigma = \frac{2 G m ((1 + \eta L)^T - (1 + \eta L)^K)}{n L}.
    \end{equation*}
    \item If $\ell$ is convex (satisfies Assumption \ref{assump:convex}), then for $\eta \leq \frac{2}{L}$ we have 
    \begin{equation*}
        \Sigma = \frac{2 \eta G m (T- K)}{n}.
    \end{equation*}
    \item If $\ell$ is $\mu$-strongly convex (satisfies Assumption \ref{assump:stronglyconvex}), then for $\eta \leq \frac{\mu}{L^2}$ and $\gamma = \sqrt{1 - \eta \mu}$ we have

    \begin{equation*}
        \Sigma = \frac{2 \eta G m (\gamma ^K - \gamma ^T)}{n \mu }.
    \end{equation*}
\end{itemize}

\end{theorem}

Theorem \ref{thm:unlearningprojected} highlights the privacy-utility-complexity tradeoff of PSGD-R2D, through the relationship between $\Sigma$, $\varepsilon$ and $K$. We observe that for all settings, $\Sigma$ decreases to zero as $K$ increases from zero to $T$, and it decays exponentially for strongly convex functions and linearly for convex functions. For general nonconvex functions, $\Sigma$ still decreases to zero as $K$ goes to $T$, but at a slower rate. This reflects the contractive, semi-contractive, and expansive properties of gradient algorithms on strongly convex, convex, and nonconvex functions respectively.

For strongly convex functions, we achieve $K$ that is \textit{better than sublinear} in $T$ for constant noise; in fact, it converges to a constant for large $T$, implying that $T - K$, the computational advantage of unlearning, is potentially \textit{infinite}. 

\begin{corollary}
\label{cor:sublinear}
    If $\ell$ is $\mu$-strongly convex (satisfies Assumption \ref{assump:stronglyconvex}), then $\Sigma$ is uniformly upper bounded as follows.
    \begin{equation}
        \Sigma \leq \frac{2 \eta G m (1 - \gamma^T)}{n \mu}.
    \end{equation}
    In addition, for constant $\Sigma$ we have
    \begin{equation*}
        K = \frac{\log( \frac{n \mu \Sigma}{2 m \eta G} + \gamma ^T ) }{\log(\gamma)} \leq \frac{\log( \frac{n \mu \Sigma}{2 m \eta G}) }{\log(\gamma)} .
    \end{equation*}
\end{corollary}
\noindent \textit{Proof.} See Appendix \ref{sec:sublinear}.

We note that $\frac{n \mu \Sigma}{2 \eta G} + \gamma^T \leq 1$, so the equation of $K$ is overall positive. For small $T$, $K$ increases at an approximately linear rate, but as $T$ grows to infinity, $K$ converges to a constant value. This implies that for a given level of of privacy $(\varepsilon, \delta)$ and noise, the number of unlearning iterations required is uniformly upper bounded even if the number of training iterations $T$ is arbitrarily large.

\subsection{SGD-R2D: Rewinding with SGD}
Now we consider rewinding without projection. A key challenge of analyzing SGD unlearning algorithms on an unbounded domain is that the noise and unlearning bias may also be unbounded. We therefore require a reasonable assumption that the second moment of the stochastic gradient is relatively bounded, as in Assumption \ref{assump:bc_assump}.

\begin{assumption}
\label{assump:bc_assump}
For all datasets $\tilde{\mathcal{D}} \sim \mathcal{Z}$, let $g_{\mathcal{B}}(\theta)$ represent the gradient estimator constructed from a batch $\mathcal{B} \sim \tilde{\mathcal{D}} $ uniformly  sampled with replacement. Then the second moment of the gradient estimator is relatively bounded with constants $B, C \geq 0$ such that for all $\theta \in \mathbb{R}^d$,
    \begin{equation}
        \mathbb{E}[\lVert g_{\mathcal{B}} (\theta) \rVert^2] \leq B \lVert \nabla \mathcal{L}_{\tilde{\mathcal{D}} }(\theta) \rVert^2 + C.
    \end{equation}
\end{assumption}

\noindent For finite-sum problems (\ref{eq:finitesum}), Assumption \ref{assump:bc_assump} is satisfied by strongly convex loss functions and functions that satisfy the Polyak--\L{}ojasiewicz (PL) inequality, as established in \cite{khaled2020bettertheorysgdnonconvex, garrigos2024handbookconvergencetheoremsstochastic, gower2021sgdstructurednonconvexfunctions}. In particular, for batch sampling with replacement, the constant $B$ can be \textit{explicitly} computed from $L$ and the batch size $b$, and the constant $C$ can be estimated from how close the model is to interpolation, where perfect interpolation (or overparametrization) implies $C = 0$.

To handle the unlearning bias, we observe that it can be bounded by the $\chi^2$-divergence between the empirical distributions of $\mathcal{D}$ and $\mathcal{D}'$, denoted as $P_{\mathcal{D}}$ and $P_{\mathcal{D}'}$ respectively (Lemma \ref{lemma:unlearningbias}). This conveniently allows us to exploit existing conditions on the noise like Assumption \ref{assump:bc_assump}. 

Finally, when we leverage the above rationale, we note that the accumulated bias and noise effects during training can be upper bounded by the loss at initialization. To achieve a dataset-independent result, we require the following assumption that the loss is finite at initialization for all $z \in \mathcal{Z}$.

\begin{assumption}
    \label{assump:boundedell} For all $z \in \mathcal{Z}$ and $\theta \in \mathbb{R}^d$, $\ell(z; \theta) \geq 0$. Moreover, the loss function is bounded at initialization $\theta_0$ by some constant $\ell_{\theta_0}$, such that for all $z \in \mathcal{Z}$, we have $\ell(z ; \theta_0) \leq \ell_{\theta_0}$. 
\end{assumption}
\noindent Assumption \ref{assump:boundedell} is satisfied if the underlying data distribution $\mathcal{Z}$ is closed and bounded.

With the above stipulations, we are able to leverage a similar proof structure as in Section \ref{sec:PSGD_main} to achieve provable SGD-R2D unlearning on unbounded domains as follows.

\begin{theorem} 
\label{thm:unlearningr2d}
Suppose that the loss function $\ell$ satisfies Assumptions \ref{assump:lipschitz}, \ref{assump:bc_assump}, and \ref{assump:boundedell} and we implement SGD-R2D (Algorithms 1 and 2) with $\mathcal{C} = \mathbb{R}^d$ and 
$\sigma = \frac{\Sigma \sqrt{2 \log (1.25/\delta)}}{\varepsilon\delta}.$ 
Then for $0 < \varepsilon \leq 1$ and $\delta > 0$, SGD-R2D is an $(\varepsilon, 2\delta)$-unlearning algorithm for the following $\Sigma$:
\begin{itemize}
    \item For general functions, we have 
    \begin{equation*}
        \Sigma = O(((1 + \eta L)^T - (1 + \eta L)^K)(T - K)^{1/2}).
    \end{equation*}
    \item If $\ell$ is convex (satisfies Assumption \ref{assump:convex}), then for $\eta \leq \frac{2}{L}$ we have 
    \begin{equation*}
        \Sigma = O((T-K)^{3/2}).
    \end{equation*}
    \item If $\ell$ is $\mu$-strongly convex (satisfies Assumption \ref{assump:stronglyconvex}), then for $\eta \leq \frac{\mu}{L^2}$ and $\gamma = \sqrt{1 - \eta \mu}$ we have

    \begin{equation*}
        \Sigma = O((\gamma ^K - \gamma ^T)(T - K)^{1/2}).
    \end{equation*}
\end{itemize}
For all statements, the $O(\cdot)$ notation hides dependencies on $B$, $C$, $\eta$, $L$, $\ell_{\theta_0}$, $m$, $n$, and $\mu$.
\end{theorem}
\noindent \textit{Proof.} See Appendix \ref{sec:SGDR2D_append}.

\subsection{D2D: Descending with SGD}

Finally, we prove that SGD-D2D achieves certified unlearning. We first show that during training, the biased SGD iterates $\{\theta_t \}$ converge to a neighborhood of $\theta^{*'}$, the global minimum of $\mathcal{L}_{\mathcal{D}'}$, and we then show that the unbiased SGD iterates $\{ \theta''_t \}$ converge closer to $\theta^{*'}$. We note that this differs  from the original D2D proof, where Lipschitz \textit{continuity} and strong convexity are combined to show that $\theta^{*'}$ is close to the global minimum of $\mathcal{L}_{\mathcal{D}}$, a property that does not hold outside of this highly constrained setting.

\begin{theorem}
\label{thm:unlearningstronglyconvex}
Suppose that $\ell$ satisfies Assumptions \ref{assump:lipschitz}, \ref{assump:stronglyconvex}, \ref{assump:bc_assump}, and \ref{assump:boundedell}. Let $\eta \leq  \frac{1}{B L}$ and $\frac{m}{n} < \frac{1}{6B + 1}$. Let 
$$T = K + \frac{\log(\ell_{\theta_0}) - \log(\frac{5C}{4 B \mu})}{\log(\frac{1}{1 - \eta \mu/2})}.$$
Then for $0 < \varepsilon \leq 1$ and $\delta > 0$, SGD-D2D (Algorithms \ref{alg:learn1} and \ref{alg:unlearn1}) is an $(\varepsilon, 2 \delta)$-certified unlearning algorithm with 
    $\sigma = \frac{\Sigma}{\varepsilon} \sqrt{\frac{2 \log (1.25/\delta)}{\delta}},$
    where
    $$\Sigma^2 =  \frac{5C}{\mu^2 B}((1 - \frac{\eta \mu}{2})^{2K} + 2(1 - \frac{\eta \mu}{2})^{K}) +  \frac{4 L \eta C}{\mu^2} .$$
\end{theorem}
\noindent \textit{Proof.} See Appendix \ref{sec:SGD-D2D-append}.

Theorem \ref{thm:unlearningstronglyconvex} shows that SGD-D2D can achieve certified unlearning as long as the proportion of unlearned data is small enough, which allows the unlearning bias to be ``folded into" the standard unbiased SGD analysis, yielding linear convergence to the global minimum. The noise decays \textit{exponentially} with the number of unlearning iterates $K$. This mirrors the original D2D result in \cite{neel21a}, but the stochastic setting yields an $O(\eta)$ constant term at the end which cannot be eliminated. Moreover, without the bounded gradient/Lipschitz smoothness assumption of \cite{neel21a}, the computation advantage of unlearning, $T - K$, depends on the initialization; in fact, if $\theta_0$ happens to be close enough to the optimum of $\mathcal{L}_{\mathcal{D}'}$, then unlearning may not be more efficient than retraining from scratch.

\section{Experiments}

\begin{figure*}[ht]
  \centering
\includegraphics[width=0.47\linewidth]{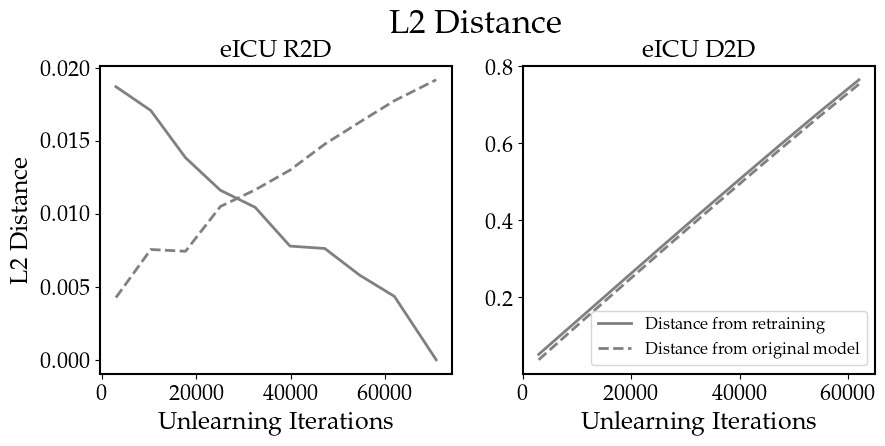}
\includegraphics[width=0.47\linewidth]{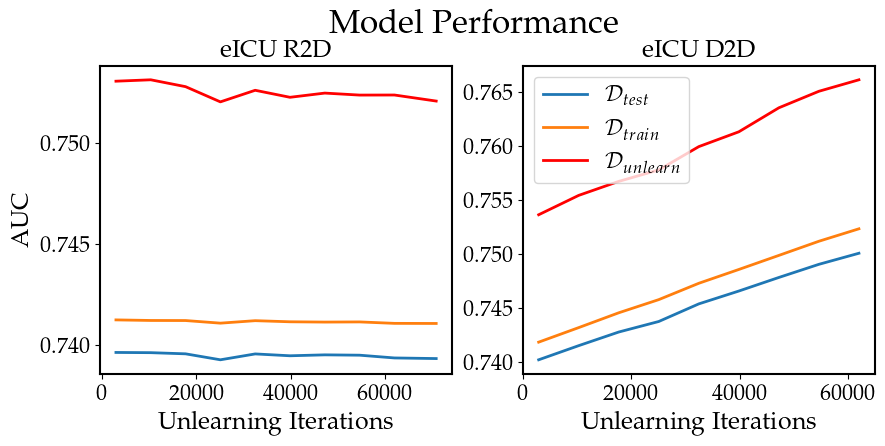}

\includegraphics[width=0.47\linewidth]{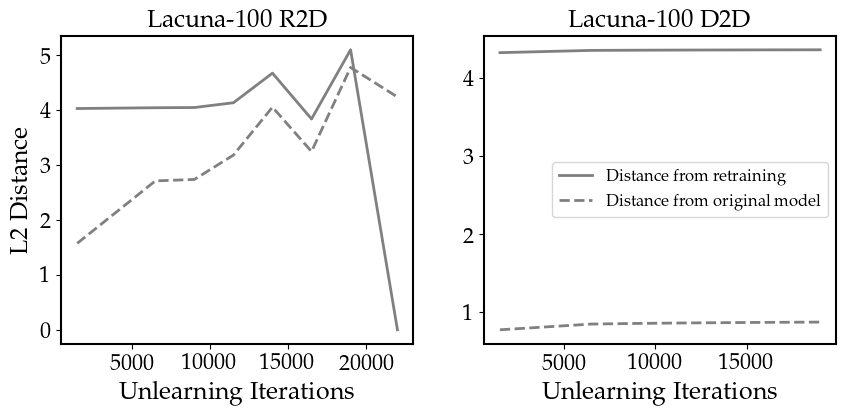}
\includegraphics[width=0.47\linewidth]{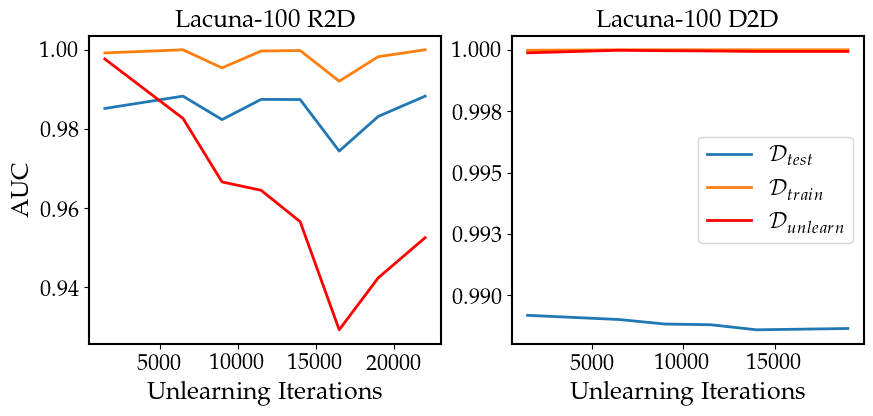}
  \caption{Comparison of unlearning capabilities of noiseless R2D and D2D. The eICU results are in the top row and the Lacuna-100 results are in the bottom row. The left two columns display the L2 distance from unlearning and the original model as the number of unlearning iterations increases. The right two columns display the model performance on the unlearned, retained, and test sets, denoted as $\mathcal{D}_{unlearn}$, $\mathcal{D}_{train}$, and $\mathcal{D}_{test}$ respectively.}
  \label{fig:unlearning_nonoise}
\vspace{-0.1in}
\end{figure*}

We follow the experimental setup in \cite{mu2025rewindtodeletecertifiedmachineunlearning} applied to PSGD instead of GD. For all experiments, we train a binary classifier with SGD and the cross-entropy loss function, and we unlearn a subset of the data from the trained model. For small-scale experiments, we use the eICU dataset \cite{Pollard_Johnson_Raffa_Celi_Mark_Badawi_2018}, a large multi-center intensive care
unit (ICU) tabular dataset, and we train a multilayer perceptron (MLP) with three hidden layers to predict whether patients stay in the hospital for longer or shorter than a week, using their intake data. For large-sclae experiments, we consider the Lacuna-100 dataset \cite{golatkar1}, constructed from the VGGFace2 dataset \cite{Cao18} and the MAAD-Face annotations \cite{terh} as described in \cite{mu2025rewindtodeletecertifiedmachineunlearning}, and we train a ResNet-18 model \cite{He_2016_CVPR} to perform binary gender classification. In contrast to prior work, our experiments focus on examining the difference between descending vs. rewinding. We considering varying $K$, or number of unlearning iterations, and we evaluate the effects of D2D and R2D at the same level of computation and batch size.  For additional experimental details, see Appendix \ref{sec:experiments}.

We evaluate unlearning efficacy using three metrics. First, we compute the L2 distance in parameter space between unlearning and retraining, where better unlearning minimizes this distance. Second, we consider the model performance on the unlearned dataset $\mathcal{D}_{unlearn}$ before and after unlearning. A  decrease in performance suggests that the model is losing information about the unlearned samples. Finally, we apply membership inference attacks (MIAs) which attempt to distinguish between unlearned data and data that has never been in the training dataset. A less successful attack, indicated by a lower Area Under the Receiver Operating Characteristic Curve (AUC), indicates that unlearning is more successful. We utilize both the classic MIA \cite{ShokriSS16} that considers the model outputs after unlearning as well as the more sophisticated unlearning MIA (MIA-U) \cite{chen_mia} which compares the model outputs before and after unlearning.

Figure \ref{fig:unlearning_nonoise} compares the first two unlearning metrics, L2 distance and model performance, for R2D and D2D. Our two experimental settings illustrate varying practical outcomes of the unlearning methods. For eICU, R2D has a reliable outcome: the unlearned model moves away from the original model and towards the retrained model. This is reflected in the L2 distance as well as the model performance on $\mathcal{D}_{unlearn}$, which decreases with increased rewinding even as the performance on the retained and test sets remain constant. In contrast, D2D causes the model to progress away from both the original and retrained models. The model performance on all data, including $\mathcal{D}_{unlearn}$, improves, suggesting that the optimization algorithm has found a new descent direction on the new loss function $\mathcal{L}_{\mathcal{D}'}$. For Lacuna-100, rewinding still reliably moves the model away from the original trained model, but the L2 distance from the retrained model does not change significantly. This may be because in the stochastic setting, R2D is only theoretically guaranteed to reduce the distance in expectation. However, R2D does still reduce the model performance on $\mathcal{D}_{unlearn}$ while keeping the performance on other sets constant. In contrast, D2D displays very little movement in parameter space or in model performance, suggesting that it is stalled at a stationary point at initialization. Moreover, a large number of descent steps risks overfitting, as shown in the decrease in performance on $\mathcal{D}_{test}$.

Figure \ref{fig:mia} displays the results of the MIAs. While the output of the MIAs can be highly nonlinear (especially for non-i.i.d. data), in general we observe that more unlearning decreases the attack success, especially for rewinding. Moreover, adding a small amount of noise tends to reduce the success of the attack. Ultimately, we find that descent-based algorithms (including D2D, \cite{chien2024stochastic}, and \cite{koloskova2023revisitinggradientclippingstochastic}) can either improve the model performance on all sets, which may obfuscate whether unlearning has occurred, or it may stall at a stationary point. On the other hand, R2D has a more reliable unlearning effect, but may not get any performance boost from unlearning.

For additional results and experimental details, including the privacy-utility-complexity tradeoff of the derived guarantees, comparisons with baseline certified unlearning methods, and anonymous GitHub repository, see Appendix \ref{sec:experiments}.

\begin{figure}[b!]
    \centering
    \includegraphics[width=0.97\linewidth]{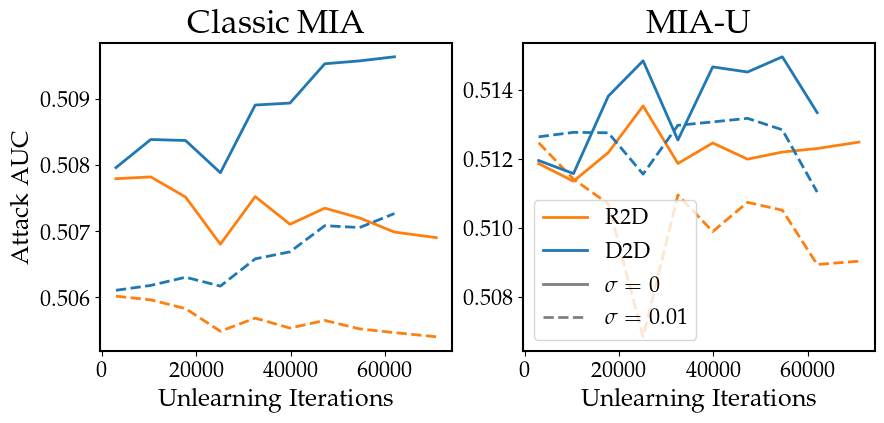}
    \includegraphics[width=0.97\linewidth]{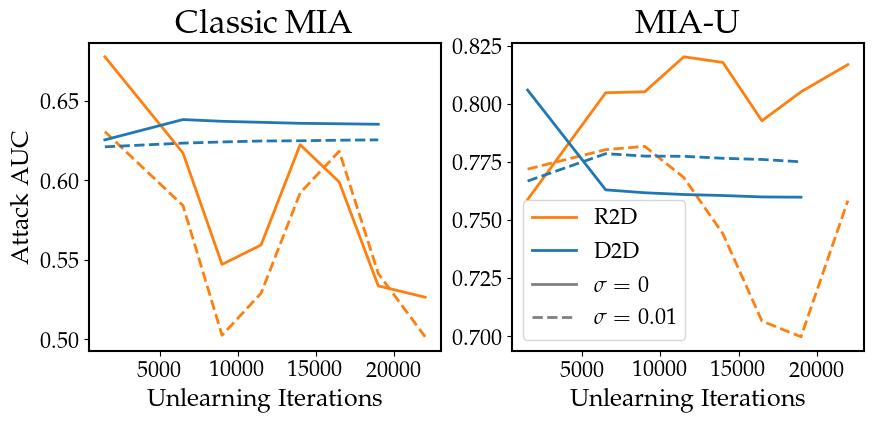}
    \caption{Comparison of MIA success for R2D and D2D. The left column displays the results of the classic MIA, and the right column displays the results of the unlearning MIA.}
    \label{fig:mia}
\end{figure}

\section{Conclusion}

In this work, we prove certified unlearning guarantees for SGD-R2D and SGD-D2D. Our analysis demonstrates that we can achieve a tighter bound for SGD-D2D on strongly convex functions, whereas SGD-R2D, which relies on the underlying contractivity of the gradient system, has a looser first-moment bound. However, SGD-R2D is always more computationally efficient than retraining. Experiments demonstrate the practical effects of each method.

\pagebreak

\section*{Impact Statement}

This paper presents work whose goal is to advance the field of Machine
Learning. There are many potential societal consequences of our work, none
which we feel must be specifically highlighted here.

\bibliography{example_paper}
\bibliographystyle{icml2026}

\newpage
\appendix
\onecolumn
\appendix

\section{Proofs}

\subsection{Overview}

In Appendix \ref{sec:helpers}, we list some helper lemmas that are standard results in probability and optimization that are useful for this work. In Appendix \ref{sec:relaxedGM_append}, we prove the theorems governing the first or second moment sensitivity bounds, which are used for all later proofs. In Appendix \ref{sec:contraction_append}, we establish the contractive, semi-contractive, and expansive properties of gradient algorithms on strongly convex, convex, and nonconvex functions respectively. These results, which reflect well-known behaviors of gradient systems, are essential for proving unlearning for PSGD-R2D in Appendix \ref{sec:PSGD-R2D_append} and SGD-R2D in Appendix \ref{sec:SGDR2D_append}. Finally, in Appendix \ref{sec:SGD-D2D-append}, we prove unlearning for SGD-D2D, under the same unbounded domain conditions of the analysis of SGD-R2D in the previous section.

\subsection{Helper lemmas}
\label{sec:helpers}

The following lemmas are standard results in probability and optimization which we utilize to obtain the results in this paper.

\begin{lemma}
Let $\mathcal{C} \subset \mathbb{R}^d$ be a nonempty, closed, and convex set. For each $\theta \in \mathbb{R}^d$, let $\Pi_{\mathcal{C}}(\theta)$ denote the (unique) projection of $\theta$ onto $\mathcal{C}$ as defined in (\ref{eq:projection}). Then the projection is nonexpansive such that for all $\theta, \theta' \in \mathbb{R}^d$,
\begin{equation}
\label{eq:projection_nonexpansive}
    \lVert \Pi_{\mathcal{C}}(\theta) - \Pi_{\mathcal{C}}(\theta') \rVert \leq \lVert \theta - \theta' \rVert. 
\end{equation}
\end{lemma}

\begin{lemma}
\label{lemma:amgm}
    For constants $a \geq 0$, $b \geq 0$, we have
\begin{align*}
\sqrt{a + b} \leq \sqrt{a} + \sqrt{b} &\leq \sqrt{2(a + b)} ,\\
\sqrt{a} + \sqrt{b} + \sqrt{c} &\leq\sqrt{3(a + b + c)}.
\end{align*}
\end{lemma}
\begin{proof}
    We can use the AM-GM inequality as follows.
    \begin{equation*}
        (\sqrt{a} + \sqrt{b})^2 = a + 2 \sqrt{ab} + b \leq 2(a + b).
    \end{equation*}
\end{proof}

\begin{lemma}
    Suppose $X$, $Y$ are independent random variables over the same domain and $\mathbb{E}[X] = 0$. Then 
    $$\mathbb{E}[\lVert X - Y \rVert^2] = \mathbb{E}[\lVert X \rVert^2] + \mathbb{E}[\lVert Y \rVert^2].$$
\end{lemma}

\begin{lemma}
\label{lemma:lipschitzsmooth}
    Suppose for all $z \in \mathcal{D}$, $\ell(z ; \theta)$ is $L$-Lipschitz smooth for all $\theta \in \mathbb{R}^d$. Then $\mathcal{L}_{\mathcal{D}}(\theta) = \frac{1}{n} \sum_{i = 1}^n \ell(z ; \theta) $ is aSlso $L$-Lipschitz smooth in $\theta$.
\end{lemma}

\begin{lemma}
    \label{lemma:towerproperty} (Tower property of expectation)
    Let $(\Omega,\mathcal{F},\mathbb{P})$ be a probability space and
 $\{\mathcal{F}_t\}_{t\in T}$ a filtration with $\mathcal{F}_s\subseteq\mathcal{F}_t\subseteq\mathcal{F}$ for $s\leq t$.
 If the random variable $X$ is integrable ($\mathbb{E}[|X|] < \infty$), then for all $s\le t$:
\[
\mathbb{E}[\mathbb{E}[X \mid \mathcal{F}_t]
\,|\,\mathcal{F}_s]=
\mathbb{E}[X \mid \mathcal{F}_s]
\quad\text{a.s.}
\]
\end{lemma}

The following lemmas pertain to functions that satisfy the Polyak--\L{}ojasiewicz (PL) inequality. PL functions can be nonconvex, and $\mu$-strongly convex functions are also $\mu$-PL.

\begin{definition}
\label{def:PL}
    Let $\mathcal{L} : \mathbb{R}^d \to \mathbb{R}$ be a differentiable function and denote $\mathcal{L}^\star := \inf_{\theta \in \mathbb{R}^d} \mathcal{L}(\theta).$
    We say that $\mathcal{L}$ satisfies the \emph{Polyak--\L{}ojasiewicz (PL) inequality} with parameter $\mu > 0$ if
    \begin{equation}
        \label{eq:PL-inequality}
        \frac{1}{2}\,\|\nabla \mathcal{L}(\theta)\|^2 
        \;\ge\; 
        \mu \bigl( \mathcal{L}(\theta) - \mathcal{L}^\star \bigr)
        \qquad \text{for all } \theta \in \mathbb{R}^d.
    \end{equation}

\end{definition}

\begin{lemma} \cite{karimiPL}
\label{lemma:qg}
Suppose a function $\mathcal{L}: \mathbb{R}^d \to \mathbb{R}$ satisfies the PL inequality (Definition \ref{def:PL}) with parameter $\mu > 0$. Let $\mathcal{X} = \arg \min_{\theta \in \Theta} \mathcal{L}(\theta)$ represent the solution set of $\mathcal{L}$, and let $\theta^* = Proj_{\mathcal{X}}(\theta) = \arg \min _{x \in \mathcal{X}} \lVert x - \theta\rVert $. Then $\mathcal{L}$ satisfies the quadratic growth condition such that 
\begin{equation}
\label{eq:qg}
\mathcal{L}(\theta) - \mathcal{L}(\theta^*) \geq \frac{\mu}{2} \lVert \theta^* - \theta \rVert^2 \qquad \text{for all } \theta \in \mathbb{R}^d.
\end{equation}
\end{lemma}

\begin{lemma}
\label{lemma:muL}
    Let $\mathcal{L}(\theta)$ be a $L$-Lipschitz smooth, $\mu$-PL function. Then $\mu \leq L$. 
\end{lemma}

\begin{proof}
    Let $\theta^*$ represent the global minima of $\mathcal{L}$. Then by Lipschitz smoothness and the fact that $\nabla \mathcal{L}(\theta^*) = 0$, we have for any $\theta \in \mathbb{R}^d$,
    $$\mathcal{L}(\theta) - \mathcal{L}(\theta^*) \leq \frac{L}{2} \lVert \theta - \theta^* \rVert^2.$$
    Then by quadratic growth (\ref{eq:qg}),
    $$\frac{\mu}{2} \lVert \theta - \theta^* \rVert^2 \leq \frac{L}{2}\lVert \theta - \theta^* \rVert^2,$$
    $$\mu \leq L.$$
\end{proof}

\subsection{Proof of Lemma \ref{lemma:indistinguishfirstsecond}}
\label{sec:relaxedGM_append}

\begin{lemma}
\label{lemma:second_moment}
    Let $x$ and $y$ be random variables over some domain $\Omega$, and suppose
    $$\mathbb{E}[\lVert x - y \rVert^2] \leq \Sigma^2.$$
    Let $\xi, \xi'$ be draws from the Gaussian distribution $\mathcal{N}(0, \sigma^2)$. Then $X = x + \xi$, $Y =y + \xi'$ are $(\varepsilon, 2 \delta)$-indistinguishable if
$$\sigma = \frac{\Sigma}{\varepsilon} \sqrt{\frac{2 \log (1.25/\delta)}{\delta}}.$$
\end{lemma}

\begin{proof}
    By Markov's inequality, for any $t > 0$ we have the tail bound
    $$\mathbb{P}[\lVert x - y \rVert > t] \leq \frac{\Sigma^2}{t^2},$$
    so for $0 < \delta < 1,$
    $$\mathbb{P}[\lVert x - y \rVert > \frac{\Sigma }{\sqrt{\delta}}] \leq \delta.$$
    Let $E$ represent the event that $\lVert x - y \rVert \leq \frac{\Sigma }{\sqrt{\delta}}$, such that $\mathbb{P}[E] > 1 - \delta$ and $\mathbb{P}[\neg E] \leq \delta$. Then when $E$ occurs, the standard Gaussian mechanism holds with sensitivity $\frac{\Sigma}{\sqrt{\delta}}$, and we have for $\sigma = \frac{\Sigma \sqrt{2 \log(1.25/\delta)}}{\sqrt{\delta} \varepsilon }$
    \begin{align*}
        \mathbb{P}[X \in S | E] \leq e^{\varepsilon} \mathbb{P}[Y \in S | E] + \delta. 
    \end{align*}
    Without loss of generality, we have for $S \subset \Omega$, 
    \begin{align*}
        \mathbb{P}[X \in S] =& \mathbb{P}[X \in S \,|\, E] \mathbb{P}[E] + \mathbb{P}[X \in S \,|\, \neg E] \mathbb{P}[\neg E]\\
        \leq& \mathbb{P}[X \in S \,|\, E] \mathbb{P}[E] + \delta\\
        \leq& (e^{\varepsilon} \mathbb{P}[Y \in S \,|\, E] + \delta)\mathbb{P}[E] + \delta \\
        \leq& e^{\varepsilon} \mathbb{P}[Y \in S \,|\, E]\mathbb{P}[E] + \delta + \delta \\
        =& e^{\varepsilon} \mathbb{P}[ E \,|\, Y \in S]\mathbb{P}[Y \in S] + 2 \delta \\
        \leq& e^{\varepsilon} \mathbb{P}[Y \in S] + 2\delta.
    \end{align*}
\end{proof}

We note that independence is not  required for Lemma \ref{lemma:second_moment} to hold, as the bounds are true for arbitrary joint distribution between $x$, $y$, $X$ and $Y$. 


\begin{lemma}
\label{lemma:first_moment}
    Let $x$ and $y$ be random variables over some domain $\Omega$, and suppose
    $$\mathbb{E}[\lVert x - y \rVert] \leq \Sigma.$$
    Let $\xi, \xi'$ be draws from the Gaussian distribution $\mathcal{N}(0, \sigma^2)$. Then $X = x + \xi$, $Y =y + \xi'$ are $(\varepsilon, 2 \delta)$-indistinguishable if
$$\sigma = \frac{\Sigma \sqrt{2 \log (1.25/\delta)}}{\varepsilon \delta}.$$
\end{lemma}

\begin{proof}
    By Markov's inequality, for any $t > 0$ we obtain the tail bound
    $$\mathbb{P}[\lVert x - y \rVert > t] \leq \frac{\Sigma}{t},$$
    so for $0 < \delta < 1$
    $$\mathbb{P}[\lVert x - y \rVert > \frac{\Sigma }{\delta}] \leq \delta.$$
\end{proof}

\subsection{Contraction of gradient systems}
\label{sec:contraction_append}

\begin{lemma}
\label{lemma:nonexpansive} 
    Suppose $\ell$ satisfies Assumption \ref{assump:lipschitz}. For some data sample $z \sim \mathcal{Z}$, we define the gradient descent map as 
    \begin{equation*}
        T_{\eta}^{z} (\theta) = \theta - \eta \nabla \ell (z; \theta).
    \end{equation*}
\begin{enumerate}
    \item For general $L$-smooth function $\ell$, we have for all $z \in \mathcal{Z}$
    \begin{equation}
    \label{eq:part1}
\lVert T^z_{\eta} (\theta) - T^z_{\eta} (\theta') \rVert \leq (1 + \eta L) \lVert \theta - \theta' \rVert.
\end{equation}
\item If $\ell$ is convex and $\eta \leq \frac{2}{L}$, we have that the map is nonexpansive such that for all $z \in \mathcal{Z}$
\begin{equation}
\label{eq:part2}
\lVert T^z_{\eta} (\theta) - T^z_{\eta} (\theta') \rVert \leq \lVert \theta - \theta' \rVert.
\end{equation}
\item If $\ell$ is $\mu$-strongly convex and $\eta \leq \frac{\mu}{L^2}$, we have that the map is contractive such that for all $z \in \mathcal{Z}$
\begin{equation}
\label{eq:part3}
\lVert T^z_{\eta} (\theta) - T^z_{\eta} (\theta') \rVert^2 \leq (1 - \mu \eta ) \lVert \theta - \theta' \rVert^2.
\end{equation}

\item Finally, let $\tilde{\mathcal{D}} \sim \mathcal{Z}$ represent a dataset of arbitrary size, and define the map
\begin{equation*}
T^{\tilde{\mathcal{D}}}_\eta(\theta) = \theta - \eta \frac{1}{|\tilde{\mathcal{D}|}} \sum_{z \in \tilde{\mathcal{D}}} \nabla \ell(z ; \theta).
\end{equation*}
If there exists $\gamma > 0$ such that $\lVert T^z_{\eta} (\theta) - T^z_{\eta} (\theta') \rVert \leq \gamma \lVert \theta - \theta' \rVert$ for all $z \in \mathcal{Z}$, then for all $\tilde{\mathcal{D}}$ we have
\begin{equation}
\label{eq:part4}
    \lVert T^{\tilde{\mathcal{D}}}_{\eta} (\theta) - T^{\tilde{\mathcal{D}}}_{\eta} (\theta') \rVert \leq \gamma \lVert \theta - \theta' \rVert.
\end{equation} 
\end{enumerate}

\end{lemma}

\begin{proof}
To prove the first part (\ref{eq:part1}), we have
\begin{align*}
    \lVert T^z_{\eta} (\theta) - T_{\eta} (\theta') \rVert = & \lVert \theta - \eta \nabla \ell(z;\theta)  - \theta' + \eta \nabla \ell(z;\theta')\rVert \\
    \leq & \lVert \theta - \theta' \rVert + \lVert \eta \nabla \ell(z;\theta) - \eta \nabla \ell(z;\theta') \rVert \\
    \overset{\text{Assumption \ref{assump:lipschitz}}}{\leq} &  \lVert \theta - \theta' \rVert + \eta L \lVert \theta - \theta' \rVert.
\end{align*}

\noindent To prove the second part (\ref{eq:part2}), we use co-coercivity of convex and $L$-smooth functions such that for $\theta, \theta' \in \Theta$
\begin{equation*}
    \frac{1}{L} \lVert \nabla \ell(z;\theta) - \nabla \ell(z;\theta') \rVert^2 \leq \langle \nabla \ell(z;\theta) - \nabla \ell(z;\theta'), \theta - \theta' \rangle . 
\end{equation*}
We have
    \begin{align*}
    \lVert \theta - \eta \nabla \ell(z;\theta) - \theta' + \eta \nabla \ell(z;\theta') \rVert^2
    =& \lVert \theta  - \theta'  \rVert^2 - 2 \eta \langle\theta  - \theta', \nabla \ell(z;\theta) - \nabla \ell(z;\theta')\rangle + \eta^2 \lVert\nabla \ell(z;\theta) -\nabla \ell(z;\theta') \rVert^2\\
    \leq& \lVert \theta  - \theta'  \rVert^2 - 2 \eta \frac{1}{L}\lVert \nabla \ell(z;\theta) -\nabla \ell(z;\theta')\rVert^2 + \eta^2  \lVert\nabla \ell(z;\theta) - \nabla \ell(z;\theta') \rVert^2 \\
    = & \lVert \theta  - \theta'  \rVert^2 - (2 \frac{\eta}{L} - \eta^2) \lVert \nabla \ell(z;\theta) - \nabla \ell(z;\theta')\rVert^2\\
    \leq &  \lVert \theta  - \theta'  \rVert^2.
\end{align*}
To prove the third part (\ref{eq:part3}), if we know $\ell$ is $\mu$-strongly convex, we know that 
\begin{equation*}
    \langle\theta  - \theta', \nabla \ell(z;\theta) - \nabla \ell(z;\theta')\rangle \geq \mu \lVert \theta - \theta' \rVert^2.
\end{equation*}
 We have
\begin{align*}
    \lVert \theta - \eta\nabla \ell(z;\theta)  - \theta' + \eta \nabla \ell(z;\theta' )  \rVert^2 \leq& \lVert \theta  - \theta'  \rVert^2 - 2\eta \mu \lVert \theta  - \theta'  \rVert^2 + \eta^2 L^2 \lVert \theta  - \theta'  \rVert^2\\
    \leq & (1 - 2 \eta \mu + \eta^2 L^2 )  \lVert \theta  - \theta'  \rVert^2 \\
    \overset{\eta \leq \frac{\mu}{L^2}}{\leq}& (1 - \eta \mu) \lVert \theta - \theta' \rVert^2. 
\end{align*}

Finally, to prove the last part (\ref{eq:part4}), we have 
\begin{align*}
    \lVert T^{\tilde{\mathcal{D}}}_{\eta} (\theta) - T^{\tilde{\mathcal{D}}}_{\eta} (\theta') \rVert = & \lVert \theta - \eta \frac{1}{|\tilde{\mathcal{D}|}} \sum_{z \in \tilde{\mathcal{D}}} \nabla \ell(z ; \theta) - \theta' + \eta \frac{1}{|\tilde{\mathcal{D}|}} \sum_{z \in \tilde{\mathcal{D}}} \nabla \ell(z ; \theta') \rVert \\
    = & \lVert  \frac{1}{|\tilde{\mathcal{D}|}} \sum_{z \in \tilde{\mathcal{D}}} \left[ \theta -  \nabla \ell(z ; \theta) - \theta' + \nabla \ell(z ; \theta')  \right] \rVert \\
    \leq & \frac{1}{|\tilde{\mathcal{D}|}} \sum_{z \in \tilde{\mathcal{D}}}  \lVert  \theta -  \nabla \ell(z ; \theta) - \theta' + \nabla \ell(z ; \theta') \rVert \\ 
    = & \frac{1}{|\tilde{\mathcal{D}|}} \sum_{z \in \tilde{\mathcal{D}}}  \lVert  T^z_\eta (\theta) - T^z_\eta (\theta')\rVert \leq \gamma \lVert \theta - \theta' \rVert. 
\end{align*}

\end{proof}

\begin{lemma}
\label{lemma:sameloss}
Suppose there exists $\gamma > 0$ such that $\lVert T^z_{\eta} (\theta) - T^z_{\eta} (\theta') \rVert \leq \gamma \lVert \theta - \theta' \rVert$ for all $z \in \mathcal{Z}$. Suppose $w_t$, $w'_t$ represent PSGD iterates (\ref{eq:PSGD}) on the same loss function $\mathcal{L}_{\mathcal{D}}$, and $\lVert w_0 -  w'_0\rVert \leq D$. Then there exists a coupling of $\{w_t \}_{t \geq 0}$, $\{w'_t\}_{t \geq 0}$ such that for all $t \geq 0$
\begin{equation*}
    \lVert w_t - w'_t \rVert \leq \gamma^t D.
\end{equation*}
Moreover, the result also applies for SGD without projection.
\end{lemma}
\begin{proof}

We can couple $\{w_t \}_{t \geq 0}$ and $\{w'_t\}_{t \geq 0}$ by choosing the same sampled mini-batch $\mathcal{B}_t$ at every step, such that we have
\begin{align*}
    \lVert w_t - w'_t \rVert = & \lVert \Pi_{\mathcal{C}}(w_{t-1} - \eta g_{\mathcal{B}_t}(w_{t-1}))  - \Pi_{\mathcal{C}}(w'_{t-1} - \eta g_{\mathcal{B}_t}(w'_{t-1})) \rVert\\
    \overset{(\ref{eq:projection_nonexpansive})}{\leq} & \lVert w_{t-1} - \eta g_{\mathcal{B}_t}(w_{t-1})  - w'_{t-1} + \eta g_{\mathcal{B}_t}(w'_{t-1}) \rVert\\
    = & \lVert w_{t-1} - \eta \frac{1}{|\mathcal{B}_t|}\sum_{z \in \mathcal{B}_t} \nabla \ell(z;w_{t-1})  - w'_{t-1} + \eta \frac{1}{|\mathcal{B}_t|}\sum_{z \in \mathcal{B}_t} \nabla \ell(z;w_{t-1}) \rVert\\
    \leq & \gamma \lVert \theta'_{T-K + t - 1} - \theta''_{t - 1} \rVert.
\end{align*}
    
\end{proof}

\subsection{PSGD-R2D}
\label{sec:PSGD-R2D_append}

In this section, we prove Theorem \ref{thm:unlearningprojected}, which establishes certified unlearning for PSGD-R2D.
We define the iterates $\{\theta_t\}_{t=0}^T$, $\{\theta'_t\}_{t=0}^T$, $\{\theta''_t\}_{t=0}^K$ as follows.
\begin{itemize}
    \item $\{\theta_t\}_{t=0}^T$ represents the PSGD (\ref{eq:PSGD}) learning iterates on $\mathcal{L}_\mathcal{D}$, starting from $\theta_0$, where $\mathcal{B}_t \sim \mathcal{D}$ and $\theta_t$ is updated as follows,
    \begin{equation}
    \label{eq:plearning}
        \theta_t = \Pi_{\mathcal{C}}(\theta_{t-1} - \eta g_{\mathcal{B}_t}(\theta_{t-1})).
    \end{equation}
    \item $\{\theta'_t\}_{t=0}^T$ represents the learning iterates on $\mathcal{L}_\mathcal{D'}$, where $\theta'_0 = \theta_0$, $\mathcal{B}_t' \sim \mathcal{D}'$, and $\theta'_t$ is updated as follows,
    \begin{equation}
    \label{eq:plearningprime}
        \theta'_t = \Pi_{\mathcal{C}}(\theta'_{t-1} - \eta g_{\mathcal{B}'_t}(\theta'_{t-1})).
    \end{equation}
    \item $\{\theta''_t\}_{t=0}^{K}$ represents the unlearning  iterates on $\mathcal{L}_{\mathcal{D}'}$, where $\theta''_0 = \theta_{T-K}$,  $\mathcal{B}'_t \sim \mathcal{D}'$, and $\theta''_t$ is updated as follows, 
    \begin{equation}
    \label{eq:plearningunlearn}
        \theta''_t = \Pi_{\mathcal{C}}(\theta''_{t-1} - \eta g_{\mathcal{B}'_t}(\theta''_{t-1})).
    \end{equation}
\end{itemize}

To prove Theorem \ref{thm:unlearningprojected}, we need to determine $\Sigma$ such that $\mathbb{E}[\lVert \theta'_t - \theta''_K \rVert] \leq \Sigma$. This result can be combined with Lemma \ref{lemma:indistinguishfirstsecond} to yield the $(\varepsilon, \delta)$ guarantee in Theorem \ref{thm:unlearningr2d}. We can determine $\Sigma$ for contracting, semi-contracting, and expanding systems, which correspond to strongly convex, convex, and nonconvex loss functions (Lemma \ref{lemma:nonexpansive}). This yields our final noise guarantees in Theorem \ref{thm:unlearningprojected}. To bound $\Sigma$, we first require the following Lemma \ref{lemma:psgdidkman} that bounds the expected distance between $\theta'_t$ and $\theta_t$.

\begin{lemma}
\label{lemma:psgdidkman}
Suppose that the loss function $\ell$ satisfies Assumptions \ref{assump:lipschitz} and within the closed, bounded, and convex set $\mathcal{C}$, the gradient of $\ell$ is uniformly bounded by some constant $G \geq 0$ such that for all $z \in \mathcal{Z}$ and $\theta \in \mathcal{C}$,
    $\lVert \nabla \ell(z; \theta) \rVert \leq G.$ Moreover, suppose for all $z \in \mathcal{Z}$ and for any $\theta, \theta' \in \mathbb{R}^d$, the gradient descent map $T^{z}_{\eta}(\theta) = \theta - \eta \nabla \ell (z ; \theta)$ satisfies the following property for some $\gamma > 0$,
    \begin{equation}
    \label{eq:shrink1}
        \lVert T^{z}_{\eta}(\theta) - T^{z}_{\eta}(\theta') \rVert \leq \gamma \lVert \theta - \theta' \rVert.
    \end{equation}
Let $\theta_t$ and $\theta'_t$ denote the training and retraining iterates as defined in (\ref{eq:plearning}) and (\ref{eq:plearningprime}) respectively. Then there exists a coupling of the iterates $\{\theta_t \}$, $\{\theta'_t \}$ such that for $t \geq 0$
    \[
\mathbb{E}[\lVert \theta_t - \theta'_t \rVert] \leq  \begin{cases}
\frac{2 \eta G m}{n} \cdot \frac{1 - \gamma^t}{ 1 - \gamma} & \text{if } \gamma \neq 1 \\
\frac{2 \eta G m t}{n} & \text{if } \gamma = 1
\end{cases}
\]
where the expectation is taken with respect to the joint distribution of $\theta_t$ and $\theta'_t$.
\end{lemma}
\begin{proof}
We describe our coupling of $\{\theta_t\}_{t=0}^T$ and $\{\theta'_t\}_{t=0}^T$. At each step $t$, we sample $b$ data samples uniformly \textit{with replacement} to form batches. Let $\mathcal{B}_t = {z_1, z_2, ..., z_b}$ and $\mathcal{B}'_t = {z'_1, z'_2, ..., z'_b}$ represent batches sampled at time $t$ from $\mathcal{D}$ and $\mathcal{D}'$ respectively. Because the batches are sampled \textit{with replacement}, we can treat each data sample as i.i.d. to one another (and drawing a particular sample does not impact the distribution of the next). Moreover, at each time step $t$, we can choose a favorable coupling of $\mathcal{B}_t$ and $\mathcal{B}'_t$ such that they contain the exact same samples except for when $\mathcal{B}_t$ contains samples from the unlearned set. Specifically, for $z_i \in \mathcal{B}_t$, if $z_i \in \mathcal{D}'$, then $z_i' = z_i$. If $z_i \notin \mathcal{D}'$, then $z'_i$ is simply some other point sampled uniformly from $\mathcal{D}'$. With this coupling, we have
\begin{align*}
    \lVert \theta_t - \theta'_t \rVert = & \lVert \Pi_{\mathcal{C}}(\theta_{t-1} - \eta g_{\mathcal{B}_t}(\theta_{t-1})) - \Pi_{\mathcal{C}}(\theta'_{t-1} - \eta g_{\mathcal{B}'_t}(\theta'_{t-1})) \rVert \\
     \overset{(\ref{eq:projection_nonexpansive})}{\leq} & \lVert \theta_{t-1} - \eta g_{\mathcal{B}_t}(\theta_{t-1})  - \theta'_{t-1} + \eta g_{\mathcal{B}'_t}(\theta'_{t-1})\rVert \\
     = & \lVert \theta_{t-1} - \eta \frac{1}{b} \sum_{i = 1}^b \nabla \ell(z_i ; \theta_{t-1})  - \theta'_{t-1} + \eta \frac{1}{b} \sum_{i = 1}^b \nabla \ell(z'_i ; \theta'_{t-1}) \rVert \\
     \leq & \lVert \theta_{t-1} - \eta \frac{1}{b} \sum_{i = 1}^b \nabla \ell(z'_i ; \theta_{t-1})  - \theta'_{t-1} + \eta \frac{1}{b} \sum_{i = 1}^b \nabla \ell(z'_i ; \theta'_{t-1}) \rVert + \frac{\eta }{b}\lVert \sum_{i = 1}^b (\nabla \ell(z_i ; \theta_{t-1})  - \nabla \ell(z'_i ; \theta_{t-1}) )  \rVert\\
     \overset{\text{(\ref{eq:shrink1}) and Lemma \ref{lemma:nonexpansive}}}{\leq} & \gamma \lVert \theta_{t-1} - \theta'_{t-1} \rVert + \frac{\eta }{b} \sum_{i = 1}^b \lVert\nabla \ell(z_i ; \theta_{t-1})  - \nabla \ell(z'_i ; \theta_{t-1})  \rVert
\end{align*}
Let $d_{i,t} = \nabla \ell(z_i ; \theta_{t-1})  - \nabla \ell(z'_i ; \theta_{t-1})$. Under the coupling described above, we have that $d_{i,t} = 0$ with probability $\frac{n-m}{n}$. When $z_i$ belongs to the unlearned set, which occurs with probability $\frac{m}{n}$, we can bound $d_{i.t}$ as
\begin{equation*}
    \lVert d_{i,t} \rVert \leq 2G,
\end{equation*}
as the gradient is uniformly bounded. Therefore, we can take expectation of both sides of our equation, with respect to the specified coupling of the batch sampling.
\begin{align*}
    \mathbb{E}[\lVert \theta_t - \theta'_t \rVert] \leq & \gamma 
    \mathbb{E}[\lVert \theta_{t-1} - \theta'_{t-1} \rVert] +  \frac{\eta}{b} \sum_{i=1}^b \mathbb{E}[\lVert d_{i,t} \rVert ] \\
    \leq & \gamma 
    \mathbb{E}[\lVert \theta_{t-1} - \theta'_{t-1} \rVert] +  \frac{2 \eta G m}{n} \\
    \leq &  \frac{2 \eta G m}{n} \sum_{\tau = 1}^{t} \gamma^{t- \tau} 
\end{align*}
\[
\mathbb{E}[\lVert \theta_t - \theta'_t \rVert] \leq  \begin{cases}
\frac{2 \eta G m}{n} \cdot \frac{1 - \gamma^t}{ 1 - \gamma} & \text{if } \gamma \neq 1 \\
\frac{2 \eta G m t}{n} & \text{if } \gamma = 1.
\end{cases}
\]
\end{proof}


Lemma \ref{lemma:psgdidkman} allows us to prove Theorem \ref{thm:psgdgamma}, which can be combined with Lemma \ref{lemma:nonexpansive} and Lemma \ref{lemma:indistinguishfirstsecond} to yield the result in Theorem \ref{thm:unlearningprojected}.

\begin{theorem}
\label{thm:psgdgamma}
Suppose that the loss function $\ell$ satisfies Assumptions \ref{assump:lipschitz} and within the closed, bounded, and convex set $\mathcal{C}$, the gradient of $\ell$ is uniformly bounded by some constant $G \geq 0$ such that for all $z \in \mathcal{Z}$ and $\theta \in \mathcal{C}$,
    $\lVert \nabla \ell(z; \theta) \rVert \leq G.$ Moreover, suppose for all $z \in \mathcal{Z}$ and for any $\theta, \theta' \in \mathbb{R}^d$, the gradient descent map $T^{z}_{\eta}(\theta) = \theta - \eta \nabla \ell (z ; \theta)$ satisfies the following property for some $\gamma > 0$,
    \begin{equation}
    \label{eq:shrink2}
        \lVert T^{z}_{\eta}(\theta) - T^{z}_{\eta}(\theta') \rVert \leq \gamma \lVert \theta - \theta' \rVert.
    \end{equation}
Let $\theta_t$, $\theta'_t$ and $\theta''_t$ denote the training, retraining, and unlearning iterates as defined in (\ref{eq:plearning}), (\ref{eq:plearningprime}) and (\ref{eq:plearningunlearn}) respectively. If $\gamma \neq 1$, there exists a coupling of  $\theta_t$, $\theta'_t$ and $\theta''_t$ such that
\begin{equation*}
    \mathbb{E}[\lVert \theta''_K - \theta'_T \rVert ] \leq \frac{2 \eta G m (\gamma^K - \gamma^T)}{n (1 - \gamma)},
\end{equation*}
and if $\gamma = 1$, then we have 
\begin{equation*}
    \mathbb{E}[\lVert \theta''_K - \theta'_T \rVert ] \leq \frac{2 \eta G m (T - K)}{n},
\end{equation*}
where the expectation is taken with respect to the resulting joint distribution of  $\theta_t$, $\theta'_t$ and $\theta''_t$.
\end{theorem}

\begin{proof}

Lemma \ref{lemma:psgdidkman} shows that the expected distance between $\theta_t$ and $\theta'_t$ is bounded as a function of $t$. So for  $\gamma \neq 1$, we have
\begin{equation*}
    \mathbb{E}[\lVert \theta_{T-K} - \theta'_{T-K} \rVert]  \leq \frac{2 \eta G m}{n} \cdot \frac{1 - \gamma^{T-K}}{ 1 - \gamma}.
\end{equation*}

From Lemma \ref{lemma:sameloss}, we have that we can choose a coupling of $\{\theta'_t\}_{t=T-K}^T$, $\{\theta''_t\}_{t=0}^K$, such that  
\begin{align*}
    \lVert \theta'_T - \theta''_K \rVert \leq & \gamma^K \lVert \theta'_{T-K} - \theta''_0 \rVert, \\
    \mathbb{E}[\lVert \theta'_T - \theta''_K \rVert] \leq & \gamma^K \mathbb{E}[\lVert \theta'_{T-K} - \theta''_0 \rVert],\\
    = & \gamma^K \mathbb{E}[\lVert \theta'_{T-K} - \theta_{T-K} \rVert],\\
    \leq & \frac{2 \eta G m}{n} \cdot \frac{\gamma^K - \gamma^{T}}{ 1 - \gamma}.
\end{align*}

The same approach can be applied to the $\gamma = 1$ case, finishing our proof.

\end{proof}

Theorem \ref{thm:unlearningprojected} follows directly from Theorem \ref{thm:psgdgamma}, Lemma \ref{lemma:indistinguishfirstsecond} and Lemma \ref{lemma:nonexpansive}.

\subsubsection{Proof of Corollary \ref{cor:sublinear}}
\label{sec:sublinear}
\begin{proof}
    
For fixed $\Sigma$, we have
\begin{align*}
    \frac{n \mu \Sigma}{2 m \eta G}  = & \gamma^K - \gamma ^T, \\
    \frac{n \mu \Sigma}{2 m \eta G} + \gamma^T = &\gamma^K \leq 1, \\
    \log( \frac{n \mu \Sigma}{2 m \eta G} + \gamma^T ) = & K \log(\gamma), \\
    K = \frac{\log( \frac{n \mu \Sigma}{2 m \eta G} + \gamma^T ) }{\log(\gamma)} \leq & \frac{\log( \frac{n \mu \Sigma}{2 m \eta G}) }{\log(\gamma)} .
\end{align*}

\end{proof}

\subsection{SGD-R2D}
\label{sec:SGDR2D_append}
In this section, we prove Theorem \ref{thm:unlearningr2d}, which establishes certified unlearning for SGD-R2D without projection or bounded domain. We establish preliminaries and present Lemmas \ref{lemma:relative_bias_bound} and \ref{lemma:sgdconvergence_nonconvex}. We then use those results to prove Theorem \ref{thm:sgdgamma}, from which Theorem \ref{thm:unlearningr2d} directly follows. In Appendix \ref{sec:relative_bias_bound}, we prove Lemma \ref{lemma:relative_bias_bound}. In Appendix \ref{sec:sgd_convergence_nonconvex}, we prove Lemma \ref{sec:sgd_convergence_nonconvex}.

We define the iterates $\{\theta_t\}_{t=0}^T$, $\{\theta'_t\}_{t=0}^T$, $\{\theta''_t\}_{t=0}^K$ as follows.
\begin{itemize}
    \item $\{\theta_t\}_{t=0}^T$ represents the SGD (\ref{eq:SGD}) learning iterates on $\mathcal{L}_\mathcal{D}$, starting from $\theta_0$, where $\mathcal{B}_t \sim \mathcal{D}$ and $\theta_t$ is updated as follows,
    \begin{equation}
    \label{eq:learning}
        \theta_t = \theta_{t-1} - \eta g_{\mathcal{B}_t}(\theta_{t-1}).
    \end{equation}
    \item $\{\theta'_t\}_{t=0}^T$ represents the SGD learning iterates on $\mathcal{L}_\mathcal{D'}$, where $\theta'_0 = \theta_0$, $\mathcal{B}_t' \sim \mathcal{D}'$ and $\theta'_t$ is updated as follows,
    \begin{equation}
    \label{eq:learningprime}
        \theta'_t = \theta'_{t-1} - \eta g_{\mathcal{B}'_t}(\theta'_{t-1}).
    \end{equation}
    \item $\{\theta''_t\}_{t=0}^{K}$ represents the SGD unlearning  iterates on $\mathcal{L}_{\mathcal{D}'}$, where $\theta''_0 = \theta_{T-K}$,  $\mathcal{B}'_t \sim \mathcal{D}'$, and $\theta''_t$ is updated as follows, 
    \begin{equation}
    \label{eq:learningunlearn}
        \theta''_t = \theta''_{t-1} - \eta g_{\mathcal{B}'_t}(\theta''_{t-1}).
    \end{equation}
\end{itemize}

To prove Theorem \ref{thm:unlearningr2d}, we need to determine $\Sigma$ such that $\mathbb{E}[\lVert \theta'_t - \theta''_K \rVert] \leq \Sigma$. This result can be combined with Lemma \ref{lemma:indistinguishfirstsecond} to yield the $(\varepsilon, \delta)$ guarantee in Theorem \ref{thm:unlearningr2d}. We can determine $\Sigma$ for contracting, semi-contracting, and expanding systems, which correspond to strongly convex, convex, and nonconvex loss functions (Lemma \ref{lemma:nonexpansive}). This yields our final noise guarantees in Theorem \ref{thm:unlearningr2d}.

As discussed in the main body, we require Lemma \ref{lemma:unlearningbias}, which bounds the unlearning bias in terms of the second moment of the stochastic noise.

\begin{lemma}
\label{lemma:unlearningbias}
For all $\theta \in \mathbb{R}^d$, the unlearning bias is bounded as follows,
    \begin{equation}\lVert \nabla \mathcal{L}_{\mathcal{D}}(\theta) - \nabla \mathcal{L}_{\mathcal{D}'}(\theta) \rVert^2 \leq \frac{m}{n - m} \mathbb{E}_{z \sim \mathcal{D}}[\lVert \nabla \ell(\theta ; z) \rVert^2]. 
    \end{equation}
\end{lemma}
\noindent\textit{Proof.} See Appendix \ref{sec:proofunlearningbias}.

\noindent In addition, we require Lemma \ref{lemma:sgdconvergence_nonconvex}, a classic result which states that assuming the noise is relatively bounded (Assumption \ref{assump:bc_assump}), the gradient norms of the \textit{unbiased} SGD iterates are bounded. Lemma \ref{lemma:unlearningbias} and \ref{lemma:sgdconvergence_nonconvex} allow us to isolate and capture the accumulated disturbances from noise and bias, which are upper bounded by a linear term in $T$ and the loss at initialization. This can be combined with Assumption \ref{assump:boundedell} to achieve a dataset-independent bound.

\begin{lemma}
\label{lemma:sgdconvergence_nonconvex}{(Convergence of SGD)}
Let $\theta_t$ represent SGD iterates (\ref{eq:SGD}) on a loss function $\mathcal{L}_{\mathcal{D}}$. 
    Suppose Assumptions \ref{assump:lipschitz} and \ref{assump:bc_assump} are satisfied. Then for $\theta_{t+1} = \theta_t - \eta g_{\mathcal{B}_t}(\theta_t)$ and $\eta \leq \frac{1}{LB}$, we have
    \begin{align*}
        \mathbb{E}[\sum_{t=0}^{T - 1} \lVert \nabla \mathcal{L}_{\mathcal{D}}(\theta_{t}) \rVert^2 ] \leq & \frac{2}{\eta}(\mathcal{L}_{\mathcal{D}}(\theta_0) - \mathcal{L}_{\mathcal{D}}^*) + L \eta C T.
    \end{align*}
\end{lemma}
\noindent\textit{Proof.} See Appendix \ref{sec:sgd_convergence_nonconvex}.

We use the above results to prove the following Lemma \ref{lemma:thetaandthetaprimer2d} that bounds the expected distance between $\theta'_t$ and $\theta_t$.

\begin{lemma}
\label{lemma:thetaandthetaprimer2d}
Suppose that the loss function $\ell$ satisfies Assumptions \ref{assump:lipschitz}, \ref{assump:bc_assump}, and \ref{assump:boundedell}. Moreover, suppose for all $z \in \mathcal{Z}$ and for any $\theta, \theta' \in \mathbb{R}^d$, the gradient descent map $T^{z}_{\eta}(\theta) = \theta - \eta \nabla \ell (z ; \theta)$ satisfies the following property for some $\gamma > 0$,
    \begin{equation}
    \label{eq:shrink}
        \lVert T^{z}_{\eta}(\theta) - T^{z}_{\eta}(\theta') \rVert \leq \gamma \lVert \theta - \theta' \rVert.
    \end{equation}
Let $\theta_t$ and $\theta'_t$ denote the training and retraining from scratch iterates as defined in (\ref{eq:learning}) and (\ref{eq:learningprime}) respectively. there exists a coupling of the iterates $\{\theta_t\}$, $\{\theta'_t\}$ such that for $t \geq 0$ and $\gamma \neq 1$, we have
\begin{equation*}
    \mathbb{E}[\lVert \theta_t - \theta'_t \rVert ] \leq \eta \frac{1 - \gamma^t}{1 - \gamma }\left(3  B \left( \frac{2}{\eta} \ell_{\theta_0} + L \eta C t \right) (\frac{3n - m}{n-m}) + 6 (\frac{4n - 3m}{n-m}) C  \right)^{1/2} .
\end{equation*}
If $\gamma = 1$, we have
\begin{equation*}
    \mathbb{E}[\lVert \theta_t - \theta'_t \rVert ] \leq  \eta t \left(3  B \left( \frac{2}{\eta} \ell_{\theta_0} + L \eta C t \right) (\frac{3n - m}{n-m}) + 6 (\frac{4n - 3m}{n-m}) C  \right)^{1/2} .
\end{equation*}
\end{lemma}
\noindent\textit{Proof.} See Appendix \ref{sec:proofthetaandthetaprimer2d}.

Lemma \ref{lemma:thetaandthetaprimer2d} allows us to prove Theorem \ref{thm:sgdgamma}, which can be combined with Lemma \ref{lemma:nonexpansive} and Lemma \ref{lemma:indistinguishfirstsecond} to yield the result in Theorem \ref{thm:unlearningr2d}.

\begin{theorem}
\label{thm:sgdgamma}
    Suppose that the loss function $\ell$ satisfies Assumptions \ref{assump:lipschitz}, \ref{assump:bc_assump}, and \ref{assump:boundedell}. Moreover, suppose for all $z \in \mathcal{Z}$ and for any $\theta, \theta' \in \mathbb{R}^d$, the gradient descent map $T^{z}_{\eta}(\theta) = \theta - \eta \nabla \ell (z ; \theta)$ satisfies the property that for some $\gamma > 0$ $\lVert T^{z}_{\eta}(\theta) - T^{z}_{\eta}(\theta') \rVert \leq \gamma \lVert \theta - \theta' \rVert$.
Let $\theta_t$, $\theta'_t$ and $\theta''_t$ denote the training, retraining, and unlearning iterates as defined in (\ref{eq:learning}), (\ref{eq:learningprime}) and (\ref{eq:learningunlearn}) respectively. If $\gamma \neq 1$, there exists a coupling of  $\theta_t$, $\theta'_t$ and $\theta''_t$ such that
\begin{equation*}
    \mathbb{E}[\lVert \theta''_K - \theta'_T \rVert ] \leq \eta \frac{\gamma^K - \gamma^T}{1 - \gamma } \left(3  B \left( \frac{2}{\eta} \ell_{\theta_0} + L \eta C (T - K) \right) (\frac{3n - m}{n-m}) + 6 (\frac{4n - 3m}{n-m}) C  \right)^{1/2},
\end{equation*}
and if $\gamma = 1$, then we have 
\begin{equation*}
    \mathbb{E}[\lVert \theta''_K - \theta'_T \rVert ] \leq \eta (T - K) \left(3  B \left( \frac{2}{\eta } \ell_{\theta_0} + L \eta C (T - K) \right) (\frac{3n - m}{n-m}) + 6 (\frac{4n - 3m}{n-m}) C  \right)^{1/2} ,
\end{equation*}
where the expectation is taken with respect to the resulting joint distribution of  $\theta_t$, $\theta'_t$ and $\theta''_t$.
\end{theorem}

\begin{proof}
    Lemma \ref{lemma:thetaandthetaprimer2d} shows that the expected distance between $\theta_t$, $\theta'_t$, is bounded as a function of $t$. So for  $\gamma \neq 1$, we have
\begin{equation*}
    \mathbb{E}[\lVert \theta_{T-K} - \theta'_{T-K} \rVert ]\leq \eta \frac{1 - \gamma^{T-K}}{1 - \gamma }\left(3  B \left( \frac{2}{\eta} \ell_{\theta_0} + L \eta C (T-K) \right) (\frac{3n - m}{n-m}) + 6 (\frac{4n - 3m}{n-m}) C  \right)^{1/2} .
\end{equation*}

From Lemma \ref{lemma:sameloss}, we have that we can choose a coupling of $\{\theta'_t\}_{t=T-K}^T$, $\{\theta''_t\}_{t=0}^K$, such that  
\begin{align*}
    \lVert \theta'_T - \theta''_K \rVert \leq & \gamma^K \lVert \theta'_{T-K} - \theta''_0 \rVert, \\
    \mathbb{E}[\lVert \theta'_T - \theta''_K \rVert] \leq & \gamma^K \mathbb{E}[\lVert \theta'_{T-K} - \theta''_0 \rVert],\\
    \leq & \eta \frac{\gamma^K - \gamma^{T}}{1 - \gamma }\left(3  B \left( \frac{2}{\eta} \ell_{\theta_0} + L \eta C (T-K) \right) (\frac{3n - m}{n-m}) + 6 (\frac{4n - 3m}{n-m}) C  \right)^{1/2}.
\end{align*}

The same approach can be applied to the $\gamma = 1$ case, finishing our proof.

\end{proof}

\subsubsection{Proof of Lemma \ref{lemma:unlearningbias}}
\label{sec:proofunlearningbias}

To bound the unlearning bias, we first require the following result from \cite{Polyanskiy_Wu_2025} (Theorem 7.26 and Example 7.4), which follows from the Radon-Nikodym theorem and Cauchy-Schwarz inequality.

\begin{lemma}
\label{lemma:radon}
Let $P$ and $Q$ be probability measures on a measurable space 
$\mathcal{X}$ such that $P \ll Q$. 
Then for any function $f : \mathcal{X} \to \mathbb{R}$, we have 
\[
\lVert \mathbb{E}_P[f] - \mathbb{E}_Q[f] \rVert^2
\;\leq\;
\chi^2(P \| Q) \;\cdot\; \mathrm{Var}_Q(f),
\]
where
\[
\chi^2(P \| Q) 
= \int \left( w(x) - 1 \right)^2 \, dQ(x)
\]
and $w(x) = \frac{dP}{dQ}(x)$ denotes the Radon-Nikodym derivative.
\end{lemma}


\noindent Now we can proceed with the proof of Lemma \ref{lemma:unlearningbias}.

\begin{proof}

In the following, we provide an explicit proof that does not require any prior knowledge for easy understanding. However, we note that the result ultimately follows from defining the empirical distributions of $\mathcal{D}$ and $\mathcal{D}'$  as $P_{\mathcal{D}}$ and $P_{\mathcal{D}'}$, and applying Lemma \ref{lemma:radon}.

Define $w(x)$ as follows
\begin{equation}
   w(x) =  \begin{cases} 
  \frac{n}{n-m} & x \in \mathcal{D}', \\
  0 & x \in \mathcal{D}\backslash \mathcal{D}'. \\
\end{cases}
\end{equation}
Then we can write $\nabla \mathcal{L}_{\mathcal{D}'}(\theta)$ in terms of $w(z_i)$ for all $z_i \in \mathcal{D}'$ as follows,
$$\nabla \mathcal{L}_{\mathcal{D}'}(\theta) = \frac{1}{n-m} \sum_{i = 1}^{n-m} \nabla \ell (\theta ; z_i) = \frac{1}{n} \sum_{i = 1}^n \nabla \ell(\theta ; z_i) w(z_i).$$
Then for 
$\nabla \mathcal{L}_{\mathcal{D}}(\theta) = \frac{1}{n} \sum_{i = 1}^{n} \nabla \ell (\theta ; z_i)$, we have
\begin{align*}
    \lVert \nabla \mathcal{L}_{\mathcal{D}}(\theta) - \nabla \mathcal{L}_{\mathcal{D}'}(\theta) \rVert^2 = & \lVert \frac{1}{n} \sum_{i = 1}^{n} \nabla \ell (\theta ; z_i) - \frac{1}{n} \sum_{i = 1}^n \nabla \ell(\theta ; z_i) w(z_i) \rVert^2 ,\\
    = & \frac{1}{n^2} \lVert \sum_{i = 1}^{n} \nabla \ell (\theta ; z_i)(1 - w(z_i) )\rVert^2.
\end{align*}
By the Cauchy-Schwarz inequality, we have
\begin{align*}
    \lVert \nabla \mathcal{L}_{\mathcal{D}}(\theta) - \nabla \mathcal{L}_{\mathcal{D}'}(\theta) \rVert^2 \leq& \frac{1}{n^2} \sum_{i = 1}^{n}\lVert  \nabla \ell (\theta ; z_i) \rVert ^2 \sum_{i = 1}^{n} (1 - w(z_i) )^2,\\
    =& \frac{1}{n} \sum_{i = 1}^{n}\lVert  \nabla \ell (\theta ; z_i) \rVert ^2 \cdot  \frac{1}{n} (m + (n-m)(1 - \frac{n}{n-m})^2),\\
    =& \mathbb{E}_{z \sim \mathcal{D}} [\lVert  \nabla \ell (\theta ; z_i) \rVert ^2] \frac{m}{n-m}.
\end{align*}
The theory underlying Lemma \ref{lemma:unlearningbias} is that we use the second moment bound to write the difference in terms of the $\chi^2$-divergence between $P_{\mathcal{D}}$ and $P_{\mathcal{D}'}$, where $w(x)$ represents a discrete version of the Radon Nikodym derivative $\frac{dP_{\mathcal{D}'}}{dP_{\mathcal{D}}}$. The $\chi^2$-divergence then evaluates to $\frac{m}{n-m}$.

\end{proof}

\subsubsection{Proof of Lemma \ref{lemma:sgdconvergence_nonconvex}}
\label{sec:sgd_convergence_nonconvex}
\begin{proof}
By Lipschitz smoothness (Lemma \ref{lemma:lipschitzsmooth}), we have
\begin{align*} 
    \mathcal{L}_{\mathcal{D}}(\theta_{t+1}) - \mathcal{L}_{\mathcal{D}} (\theta_{t}) &\leq \nabla \mathcal{L}_{\mathcal{D}}(\theta_{t})^T (- \eta \nabla \mathcal{L}_{\mathcal{D}}(\theta_t) + \eta \xi_t) + \frac{L \eta^2}{2} \lVert g(\theta_t) \rVert^2 \\
    \mathbb{E}[\mathcal{L}_{\mathcal{D}} (\theta_{t+1})] - \mathcal{L}_{\mathcal{D}} (\theta_{t}) &\leq \nabla \mathcal{L}_{\mathcal{D}}(\theta_{t})^T (- \eta \nabla \mathcal{L}_{\mathcal{D}}(\theta_t)) + \frac{L \eta^2}{2} \mathbb{E}[\lVert  g(\theta_t) \rVert^2] \\
    &= - \eta \lVert \nabla \mathcal{L}_{\mathcal{D}}(\theta_{t}) \rVert^2 + \frac{L \eta^2}{2} \mathbb{E}[\lVert  g(\theta_t) \rVert^2]. 
\end{align*}
By Assumption \ref{assump:bc_assump} and $\eta \leq \frac{1}{LB}$, we have
\begin{align*} 
    \mathbb{E}[\mathcal{L}_{\mathcal{D}} (\theta_{t+1})] - \mathcal{L}_{\mathcal{D}} (\theta_{t}) \leq & - \eta \lVert \nabla \mathcal{L}_{\mathcal{D}}(\theta_{t}) \rVert^2 + \frac{L \eta^2}{2} (B \lVert \nabla \mathcal{L} (\theta_t) \rVert^2 + C)\\
    \leq & - \frac{\eta}{2} \lVert \nabla \mathcal{L}_{\mathcal{D}}(\theta_{t}) \rVert^2 + \frac{L \eta^2}{2} C\\
    \frac{\eta}{2} \lVert \nabla \mathcal{L}_{\mathcal{D}}(\theta_{t}) \rVert^2 \leq & \mathcal{L}_{\mathcal{D}} (\theta_{t}) - \mathbb{E}[\mathcal{L}_{\mathcal{D}'} (\theta_{t+1})] + \frac{L \eta^2 C}{2} \\
    \frac{\eta}{2} \mathbb{E}[\lVert \nabla \mathcal{L}_{\mathcal{D}}(\theta_{t}) \rVert^2] \leq & \mathbb{E}[\mathbb{E}[\mathcal{L}_{\mathcal{D}} (\theta_{t})] - \mathbb{E}[\mathbb{E}[\mathcal{L}_{\mathcal{D}'} (\theta_{t+1})] + \frac{L \eta^2 C}{2} \\
    \frac{\eta}{2} \mathbb{E}[\sum_{t=0}^{T - 1} \lVert \nabla \mathcal{L}_{\mathcal{D}}(\theta_{t}) \rVert^2 ] \leq & \mathcal{L}_{\mathcal{D}}(\theta_0) - \mathcal{L}_{\mathcal{D}}^* + \frac{L \eta^2 C}{2} T .
\end{align*}

\end{proof}

\subsubsection{Proof of Lemma \ref{lemma:thetaandthetaprimer2d}}
\label{sec:proofthetaandthetaprimer2d}
\begin{proof}
We can decompose $\lVert \theta_t - \theta'_t \rVert$ as follows,
\begin{align*}
    \mathbb{E}[\lVert \theta_{t} - \theta'_t \rVert ] \leq & \underbrace{\lVert \theta_{t-1} - \eta \nabla \mathcal{L}_{\mathcal{D}'}(\theta_{t-1}) - \theta_{t-1}' + \eta \nabla \mathcal{L}_{\mathcal{D}'}(\theta'_{t-1}) \rVert}_{(1)} + \eta \underbrace{\lVert \nabla \mathcal{L}_{\mathcal{D}'}(\theta_{t-1}) - \nabla \mathcal{L}_{\mathcal{D}}(\theta_{t-1}) \rVert}_{(2)}  \\
    &+ \eta \underbrace{\mathbb{E}[\lVert g_{\mathcal{D}'}(\theta'_{t-1}) - \nabla \mathcal{L}_{\mathcal{D}'}(\theta_{t-1}') \rVert + \lVert g_{\mathcal{D}}(\theta_{t-1}) - \nabla \mathcal{L}_{\mathcal{D}}(\theta_{t-1}) \rVert] }_{(3)},
\end{align*}
where (2) represents the unlearning bias and (3) represents the noise.
By (\ref{eq:shrink}) and (\ref{eq:part4}) of Lemma \ref{lemma:nonexpansive}, we have
$$(1) \leq \gamma \lVert \theta_{t-1} - \theta'_{t-1} \rVert.$$
By our bias bound (Lemma \ref{lemma:unlearningbias}) and relative noise bound assumption (Assumption \ref{assump:bc_assump}), we have 
$$(2) \leq (\frac{m}{n-m} \mathbb{E}_{z \sim \mathcal{D}}[\lVert\nabla \ell(z ; \theta_{t-1}) \rVert^2 ])^{1/2} \leq  (\frac{m}{n-m} (B \lVert \nabla \mathcal{L}_{\mathcal{D}} (\theta_{t-1}) \rVert^2 + C))^{1/2}.$$
We can also use Assumption \ref{assump:bc_assump} to bound the noise terms in $(3)$ as follows,
$$(3) \leq (B \lVert \nabla \mathcal{L}_{\mathcal{D}'} (\theta'_{t-1}) \rVert^2 + C)^{1/2} + (B \lVert \nabla \mathcal{L}_{\mathcal{D}} (\theta_{t-1}) \rVert^2 + C)^{1/2}.$$
Combining the above bounds and utilizing the fact that $\sqrt{a + b} \leq \sqrt{a} + \sqrt{b}$ and  Lemma \ref{lemma:amgm}, we have
\begin{align*}
    \mathbb{E}[\lVert \theta_{t} - \theta'_t \rVert] \leq &  \gamma \lVert  \theta_{t-1} - \theta'_{t-1} \rVert + \eta \sqrt{B} \lVert \nabla \mathcal{L}_{\mathcal{D}'}(\theta'_{t-1}) \rVert +  \eta \sqrt{2 (\frac{n}{n-m}) B}  \lVert \nabla \mathcal{L}_{\mathcal{D}}(\theta_{t-1}) \rVert + \eta \sqrt{2(\frac{4n - 3m}{n-m})C} \\
    \mathbb{E}[\lVert \theta_{t} - \theta'_{t} \rVert] \leq & \eta \mathbb{E}\left[\sum_{\tau=1}^{t} \gamma ^{t- \tau} \left(  \sqrt{B} \lVert \nabla \mathcal{L}_{\mathcal{D}'}(\theta'_{\tau-1}) \rVert +  \sqrt{2 (\frac{n}{n-m})B}  \lVert \nabla \mathcal{L}_{\mathcal{D}}(\theta_{\tau-1}) \rVert +  \sqrt{2(\frac{4n - 3m}{n-m})C} \right)\right] \\
     \overset{Cauchy-Schwarz}{\leq} & \eta \sqrt{\sum_{\tau=1}^{t} \gamma ^{2(t-\tau)}} \left( \sqrt{B \sum_{\tau=1}^{t}  \mathbb{E}[\lVert \nabla \mathcal{L}_{\mathcal{D}'}(\theta'_{\tau-1}) \rVert^2] } + \sqrt{2 (\frac{n}{n-m}) B\sum_{\tau=1}^{t} \mathbb{E}[\lVert \nabla \mathcal{L}_{\mathcal{D}}(\theta_{\tau-1}) \rVert^2]} \right)  \\
     & +   \eta  \sqrt{2 (\frac{4n - 3m}{n-m}) C}\sum_{\tau=1}^{t} \gamma ^{t - \tau}.
     \end{align*}
     We observe that we can simplify the geometric sums such that
     \begin{equation*}
        \sqrt{\sum_{\tau=1}^{t} \gamma ^{2(t - \tau)}} \leq \sum_{\tau=1}^{t} \gamma ^{t - \tau} = \sum_{\tau = 0}^{t-1} \gamma^{\tau}.
     \end{equation*}
This is equal to $\frac{1 - \gamma^t}{1 - \gamma}$ when $\gamma \neq 1$ and $t$ when $\gamma = 1$. 
Plugging this back into the original equation yields
\begin{align*}
     \mathbb{E}[\lVert \theta_{t} - \theta'_{t} \rVert] \leq &  \eta \sum_{\tau = 0}^{t-1} \gamma^{\tau}\left( \sqrt{B \sum_{\tau=1}^{t}  \mathbb{E}[\lVert \nabla \mathcal{L}_{\mathcal{D}'}(\theta'_{\tau-1}) \rVert^2] } + \sqrt{2 (\frac{n}{n-m}) B\sum_{\tau=1}^{t} \mathbb{E}[\lVert \nabla \mathcal{L}_{\mathcal{D}}(\theta_{\tau-1}) \rVert^2]} + \sqrt{2 (\frac{4n - 3m}{n-m}) C} \right),\\
     \overset{\text{Lemma \ref{lemma:sgdconvergence_nonconvex}}}{\leq} &  \eta \sum_{\tau = 0}^{t-1} \gamma^{\tau} \Bigg( \sqrt{B \left( \frac{2}{\eta} (\mathcal{L}_{\mathcal{D}'}(\theta_0) - \mathcal{L}_{\mathcal{D}'}^*) + L \eta C t \right)} + \sqrt{2 (\frac{n}{n-m}) B \left( \frac{2}{\eta } (\mathcal{L}_{\mathcal{D}}(\theta_0) - \mathcal{L}_{\mathcal{D}}^*) +  L \eta C t \right)} \\
     &+ \sqrt{2 (\frac{4n - 3m}{n-m}) C} \Bigg) \\
     \overset{\text{Assumption \ref{assump:boundedell}}}{\leq} &\eta \sum_{\tau = 0}^{t-1} \gamma^{\tau} \left( \sqrt{B \left( \frac{2}{\eta} \ell_{\theta_0} + L \eta C t \right)} + \sqrt{2 (\frac{n}{n-m}) B \left( \frac{2}{\eta} \ell_{\theta_0} + L \eta C t\right)} + \sqrt{2 (\frac{4n - 3m}{n-m}) C} \right) \\
     \overset{\text{Lemma \ref{lemma:amgm}}}{\leq} &\eta\sum_{\tau = 0}^{t-1} \gamma^{\tau} \left(3 \left( B \left( \frac{2}{\eta} \ell_{\theta_0} + L \eta C t \right) (1 + 2(\frac{n}{n-m})) + 2 (\frac{4n - 3m}{n-m}) C \right) \right)^{1/2} \\
     = &\eta \sum_{\tau = 0}^{t-1} \gamma^{\tau} \left(3 \left( B \left( \frac{2}{\eta} \ell_{\theta_0} + L \eta C t \right) (\frac{3n - m}{n-m}) + 2 (\frac{4n - 3m}{n-m}) C \right) \right)^{1/2} \\
     = &\eta\sum_{\tau = 0}^{t-1} \gamma^{\tau} \left(3  B \left( \frac{2}{\eta} \ell_{\theta_0} + L \eta C t\right) (\frac{3n - m}{n-m}) + 6 (\frac{4n - 3m}{n-m}) C  \right)^{1/2}. 
\end{align*}

\end{proof}

\subsection{SGD-D2D}
\label{sec:SGD-D2D-append}
In this section, we prove Theorem \ref{thm:unlearningstronglyconvex}, which establishes certified unlearning for SGD-D2D on strongly convex functions. We first establish preliminaries and provide an overall proof sketch. In Appendix \ref{sec:relative_bias_bound} we prove Lemma \ref{lemma:relative_bias_bound}, and in Appendix \ref{sec:descent} we prove Lemma \ref{lemma:descent}, which are necessary for achieving the final result. Finally, in Appendix \ref{sec:unlearningstronglyconvex}, we finish the proof of Theorem \ref{thm:unlearningstronglyconvex}.

We define the iterates $\{\theta_t\}_{t=0}^T$, $\{\theta'_t\}_{t=0}^T$, $\{\theta''_t\}_{t=0}^K$ as follows.
\begin{itemize}
    \item $\{\theta_t\}_{t=0}^T$ represents the SGD (\ref{eq:SGD}) learning iterates on $\mathcal{L}_\mathcal{D}$, starting from $\theta_0$, where $\mathcal{B}_t \sim \mathcal{D}$ and $\theta_t$ is updated as follows,
    \begin{equation}
    \label{eq:d2dlearning}
        \theta_t = \theta_{t-1} - \eta g_{\mathcal{B}_t}(\theta_{t-1}).
    \end{equation}
    \item $\{\theta'_t\}_{t=0}^T$ represents the SGD learning iterates on $\mathcal{L}_\mathcal{D'}$, where $\theta'_0 = \theta_0$, $\mathcal{B}_t' \sim \mathcal{D}'$, and $\theta'_t$ is updated as follows,
    \begin{equation}
    \label{eq:d2dlearningprime}
        \theta'_t = \theta'_{t-1} - \eta g_{\mathcal{B}'_t}(\theta'_{t-1}).
    \end{equation}
    \item $\{\theta''_t\}_{t=0}^{K}$ represents the SGD unlearning  iterates on $\mathcal{L}_{\mathcal{D}'}$, where $\theta''_0 = \theta_{T}$, $\mathcal{B}'_t \sim \mathcal{D}'$, and $\theta''_t$ is updated as follows, 
    \begin{equation}
    \label{eq:d2dlearningunlearn}
        \theta''_t = \theta''_{t-1} - \eta g_{\mathcal{B}'_t}(\theta''_{t-1}).
    \end{equation}
\end{itemize}

The first step of the proof is to  show that during training the biased SGD iterates $\{\theta_t \}_{t=0}^T$ will contract to be within some neighborhood of $\theta^{*'}$, the optimum of $\mathcal{L}_{\mathcal{D}'}$ (Lemma \ref{lemma:descent}). This result replicates the general result in \cite{AjalloeianBiasedSGD} showing linear convergence of biased SGD as long as the following conditions hold: i) the loss function is PL and Lipschitz smooth, ii) the noise satisfies a (relative) bound as in Assumption \ref{assump:bc_assump}, and iii)  there exists constants $D \geq 0$, $0 \leq  M < 1$, such that the bias at time step $t$, denoted as $d_t$, is relatively bounded as follows,
\begin{equation}
    \label{eq:otherpeopl}
    \lVert d_t \rVert^2 \leq M \lVert \nabla \mathcal{L}_{\mathcal{D}'}(\theta) \rVert^2 + D.
\end{equation}

We show that we satisfy the bias bound condition (\ref{eq:otherpeopl}) if the proportion of unlearned data is small enough in Lemma \ref{lemma:relative_bias_bound}.

\begin{lemma}
\label{lemma:relative_bias_bound} Suppose that Assumption \ref{assump:bc_assump} holds and $\frac{m}{n} \leq \frac{1}{6B + 1}$. Then for all $\theta \in \mathbb{R}^d$, the unlearning bias is bounded as follows,
    \begin{equation}
\label{eq:relative_bias_bound}
    \lVert \nabla \mathcal{L}_{\mathcal{D}}(\theta) - \nabla \mathcal{L}_{\mathcal{D}'}(\theta) \rVert^2 \leq \frac{1}{2} \lVert \nabla \mathcal{L}_{\mathcal{D}'}(\theta)\rVert^2 + \frac{C}{4B}.
\end{equation}
\end{lemma}
\noindent \textit{Proof.} See Appendix \ref{sec:relative_bias_bound}.

\begin{lemma}
\label{lemma:descent} 
    Consider the SGD algorithm on $\mathcal{L}_{\mathcal{D}}$ as defined in (\ref{eq:d2dlearning}), and let $\mathcal{L}_{\mathcal{D}'}^*$ represent the minimum value of $\mathcal{L}_{\mathcal{D}'}$. Suppose that $\eta \leq  \frac{1}{B L}$ and $\frac{m}{n} < \frac{1}{6B + 1}$. Then we have
$$\mathbb{E}[\mathcal{L}_{\mathcal{D}'}(\theta_{T})] - \mathcal{L}_{\mathcal{D}'}^* \leq (1 - \frac{\eta \mu}{2} )^T(\mathcal{L}_{\mathcal{D}'}(\theta_0)- \mathcal{L}_{\mathcal{D}'}^*) + \frac{C }{4 B \mu} + \frac{L \eta C}{\mu},$$
where the expectation is taken with respect to the underlying randomization of $\{\theta_t\}_{t=0}^T$.
\end{lemma}
\noindent \textit{Proof.} See Appendix \ref{sec:descent}.


The second step is to show that the \textit{unbiased} SGD iterates $\{\theta'_t \}_{t=0}^T$ also converge close to $\theta^{*'}$ through the classic convergence analysis. Finally, the last step is to show that during unlearning, the \textit{unbiased} SGD iterates $\{ \theta_{t}''\}_{t=0}^{K}$ will converge closer $\theta^{*'}$, closing the gap up to some neighborhood determined by the stochasticity. By tracking the progress of the loss value $\mathcal{L}_{\mathcal{D}'}$ and leveraging the quadratic growth condition (\ref{eq:qg}), we achieve a second-moment bound on $\lVert \theta'_T - \theta''_K \rVert$, which can be combined with Lemma \ref{lemma:indistinguishfirstsecond} to yield $(\varepsilon, 
\delta)$-indistinguishability.  In Appendix \ref{sec:unlearningstronglyconvex}, we carry out these proof components to achieve the result in Theorem \ref{thm:unlearningstronglyconvex}. This approach is unique to strongly convex functions and cannot be achieved for convex and nonconvex functions.

\subsubsection{Proof of Lemma \ref{lemma:relative_bias_bound}}
\label{sec:relative_bias_bound}

\begin{proof}

Combining Lemma \ref{lemma:unlearningbias} with Assumption \ref{assump:bc_assump}, we have 
\begin{align*}
    \lVert \nabla \mathcal{L}_{\mathcal{D}}(\theta) - \nabla \mathcal{L}_{\mathcal{D}'}(\theta) \rVert^2 \leq& \frac{m}{n - m} (B \lVert  \nabla \mathcal{L}_{\mathcal{D}}(\theta) \rVert^2 + C),\\
    =& \frac{m}{n - m} (B \lVert  \nabla \mathcal{L}_{\mathcal{D}}(\theta) - \nabla \mathcal{L}_{\mathcal{D}'}(\theta) + \nabla \mathcal{L}_{\mathcal{D}'}(\theta) \rVert^2 +C),\\
    \overset{\lVert a + b \rVert^2 \leq 2 \lVert a \rVert^2 + 2 \lVert b \rVert^2}{\leq}&\frac{m}{n - m} (2 B \lVert  \nabla \mathcal{L}_{\mathcal{D}'}(\theta) \rVert^2 + 2 B \lVert \nabla \mathcal{L}_{\mathcal{D}}(\theta) - \nabla \mathcal{L}_{\mathcal{D}'}(\theta) \rVert^2 + C)\\
    (1 - \frac{2 B m }{n-m})\lVert \nabla \mathcal{L}_{\mathcal{D}}(\theta) - \nabla \mathcal{L}_{\mathcal{D}'}(\theta) \rVert^2 \leq&\frac{2 B m }{n-m }  \lVert  \nabla \mathcal{L}_{\mathcal{D}'}(\theta) \rVert^2 + \frac{C m}{n-m}\\
    \lVert \nabla \mathcal{L}_{\mathcal{D}}(\theta) - \nabla \mathcal{L}_{\mathcal{D}'}(\theta) \rVert^2 \leq&\frac{2B m }{n - m - 2Bm}  \lVert  \nabla \mathcal{L}_{\mathcal{D}'}(\theta) \rVert^2 + \frac{C m }{n - m - 2Bm},
\end{align*}
where in the last step, the inequality is maintained after dividing both sides by $1 - \frac{2Bm}{n - m}$, which is positive since $\frac{m}{n} < \frac{1}{6B + 1} \leq \frac{1}{2B + 1}$. We note that $\frac{2 B m}{n - m - 2 B m} \leq \frac{1}{2}$, which can be used to simplify to the final result.

\end{proof}

\subsubsection{Proof of Lemma \ref{lemma:descent}}
\label{sec:descent}

\begin{proof}
We define the bias term $d_t$ and zero-mean noise term $\xi_t$ as
\begin{align*}
    d_t &=  \nabla \mathcal{L}_{\mathcal{D}'} (\theta_t) - \nabla \mathcal{L}_{\mathcal{D}} (\theta_t)\\
    \xi_t &= \nabla \mathcal{L}_{\mathcal{D}} (\theta_t) - g(\theta_t).
\end{align*}
Let $\{\mathcal{F}_t\}_{t\geq 0} = \{\sigma(\xi_0,...,\xi_{t-1}) \}_{t \geq 0}$ denote the natural filtration adapted to $\xi_t$ such that $\theta_t$ is $\mathcal{F}_t$-measurable and we have
$$\mathbb{E}[\xi_t \,|\, \mathcal{F}_t] = \mathbb{E}[\nabla \mathcal{L}_{\mathcal{D}}(\theta_t) - g(\theta_t) \,|\, \mathcal{F}_t] = 0.$$

By Lipschitz smoothness (Lemma \ref{lemma:lipschitzsmooth}), we have
\begin{align*} 
    \mathcal{L}_{\mathcal{D}'}(\theta_{t+1}) - \mathcal{L}_{\mathcal{D}'} (\theta_{t}) \leq & \nabla \mathcal{L}_{\mathcal{D}'}(\theta_{t})^T (- \eta \nabla \mathcal{L}_{\mathcal{D}'}(\theta_t) + \eta d_t + \eta \xi_t) + \frac{L \eta^2}{2} \lVert g(\theta_t) \rVert^2, \\
   \mathbb{E}[\mathcal{L}_{\mathcal{D}'} (\theta_{t+1}) \,|\, \mathcal{F}_t] - \mathcal{L}_{\mathcal{D}'} (\theta_{t}) \leq & \nabla \mathcal{L}_{\mathcal{D}'}(\theta_{t})^T (- \eta \nabla \mathcal{L}_{\mathcal{D}'}(\theta_t) + \eta d_t) + \frac{L \eta^2}{2} \mathbb{E}[\lVert  g(\theta_t) \rVert^2 \,|\, \mathcal{F}_t],\\
    \overset{\text{Assumption \ref{assump:bc_assump}}}{\leq} & \nabla \mathcal{L}_{\mathcal{D}'}(\theta_{t})^T (- \eta \nabla \mathcal{L}_{\mathcal{D}'}(\theta_t) + \eta d_t) + \frac{L \eta^2}{2} (B \lVert \nabla \mathcal{L}_{\mathcal{D}}(\theta_t) \rVert^2 +  C), \\
    = & \nabla \mathcal{L}_{\mathcal{D}'}(\theta_{t})^T (- \eta \nabla \mathcal{L}_{\mathcal{D}'}(\theta_t) + \eta d_t) + \frac{L \eta^2}{2} (B \lVert \nabla \mathcal{L}_{\mathcal{D}'}(\theta_t) + d_t \rVert^2 +  C),\\
    = & - \eta \nabla \mathcal{L}_{\mathcal{D}'}(\theta_{t})^T ( \nabla \mathcal{L}_{\mathcal{D}'}(\theta_t) +  d_t) + \frac{L \eta^2 B}{2} \lVert  \nabla \mathcal{L}_{\mathcal{D}'}(\theta_t) + d_t \rVert^2 + \frac{L \eta^2 C}{2}, \\
    \overset{\eta \leq \frac{1}{BL}}{\leq} & - \eta \nabla \mathcal{L}_{\mathcal{D}'}(\theta_{t})^T ( \nabla \mathcal{L}_{\mathcal{D}'}(\theta_t) +  d_t) + \frac{\eta}{2} \lVert  \nabla \mathcal{L}_{\mathcal{D}'}(\theta_t) + d_t \rVert^2 + \frac{L \eta^2 C}{2}.
    \end{align*}
    By the fact that $\lVert a + b \rVert^2  - 2 a^T (a + b) = - \lVert a \rVert^2 + \lVert b \rVert^2$, the bias cross term $\nabla \mathcal{L}_{\mathcal{D}'}(\theta_t)^T d_t$ can get ``folded" into the norm squared of $d_t$, which is bounded away with Lemma \ref{lemma:relative_bias_bound}. We have
    \begin{align*}
    \mathbb{E}[\mathcal{L}_{\mathcal{D}'} (\theta_{t+1}) \,|\, \mathcal{F}_t] - \mathcal{L}_{\mathcal{D}'} (\theta_{t}) \leq & \frac{\eta}{2} (- 2\nabla \mathcal{L}_{\mathcal{D}'}(\theta_{t})^T ( \nabla \mathcal{L}_{\mathcal{D}'}(\theta_t) +  d_t) + \lVert  \nabla \mathcal{L}_{\mathcal{D}'}(\theta_t) + d_t \rVert^2) + \frac{L \eta^2 C}{2}, \\
    =& -\frac{\eta}{2} \lVert \nabla \mathcal{L}_{\mathcal{D}'}(\theta_t) \rVert^2 + \frac{\eta}{2} \lVert d_t \rVert^2 + \frac{L \eta^2 C}{2}.
\end{align*}

Assuming that $\frac{m}{n} \leq \frac{1}{1 + 6B}$, we can bound the bias term by  Lemma \ref{lemma:relative_bias_bound} as follows,
\begin{align*}
    \mathbb{E}[\mathcal{L}_{\mathcal{D}'} (\theta_{t+1}) \,|\, \mathcal{F}_t] - \mathcal{L}_{\mathcal{D}'} (\theta_{t}) \overset{ \text{Lemma \ref{lemma:relative_bias_bound}}}{\leq}& -\frac{\eta}{2} \lVert \nabla \mathcal{L}_{\mathcal{D}'}(\theta_t) \rVert^2 + \frac{\eta}{2} (\frac{1}{2} \lVert \mathcal{L}_{\mathcal{D}'}(\theta)\rVert^2 + \frac{C}{4B}) + \frac{L \eta^2 C}{2},\\
    \leq& -\frac{\eta}{4} \lVert \nabla \mathcal{L}_{\mathcal{D}'}(\theta_t) \rVert^2 + \frac{ \eta C }{8B} + \frac{L \eta^2 C}{2}.
\end{align*}
We can use the Polyak--\L{}ojasiewicz (PL) property of strongly convex functions (Assumption \ref{assump:stronglyconvex} and Definition \ref{def:PL}) to bound in terms of $\mathcal{L}^*_{\mathcal{D}'} = \inf_{\theta \in \mathbb{R}^d} \mathcal{L}_{\mathcal{D}'}(\theta)$.
\begin{align*}
     \mathbb{E}[\mathcal{L}_{\mathcal{D}'} (\theta_{t+1}) \,|\, \mathcal{F}_t] - \mathcal{L}_{\mathcal{D}'} (\theta_{t})  \overset{\text{Assumption \ref{assump:stronglyconvex}}}{\leq}& -\frac{\eta}{2} \mu (\mathcal{L}_{\mathcal{D}'}(\theta_t)- \mathcal{L}_{\mathcal{D}'}^*) +  \frac{\eta C }{8B} + \frac{L \eta^2 C}{2},\\
    \mathbb{E}[\mathcal{L}_{\mathcal{D}'} (\theta_{t+1}) \,|\, \mathcal{F}_t] - \mathcal{L}_{\mathcal{D}'}^* \leq& (1 -\frac{\eta \mu}{2} )(\mathcal{L}_{\mathcal{D}'}(\theta_t)- \mathcal{L}_{\mathcal{D}'}^*) + \frac{\eta C }{8B} + \frac{L \eta^2 C}{2},\\
    \mathbb{E}[\mathcal{L}_{\mathcal{D}'} (\theta_{t})] - \mathcal{L}_{\mathcal{D}'}^* \overset{\text{Lemma \ref{lemma:towerproperty}}}{\leq}& (1 - \frac{\eta \mu }{2})^t(\mathcal{L}_{\mathcal{D}'}(\theta_0)- \mathcal{L}_{\mathcal{D}'}^*) 
    +  \eta (\frac{C }{8B} + \frac{L \eta C}{2}) \sum_{i=0}^{t-1} (1 - \frac{\eta \mu}{2})^i,\\
    \leq& (1 - \frac{\eta \mu}{2} )^t(\mathcal{L}_{\mathcal{D}'}(\theta_0)- \mathcal{L}_{\mathcal{D}'}^*) +  \frac{2}{\mu } (\frac{C }{8B} + \frac{L \eta C}{2}),\\
    \leq& (1 - \frac{\eta \mu}{2} )^t(\mathcal{L}_{\mathcal{D}'}(\theta_0)- \mathcal{L}_{\mathcal{D}'}^*) + \frac{C }{4B \mu} + \frac{L \eta C}{\mu},
\end{align*}
where in the second to last step we upper bound the geometric series $\sum_{i=0}^{t-1} (1 - \frac{\eta \mu}{2})^i$.
\end{proof}

\subsubsection{Proof of Theorem \ref{thm:unlearningstronglyconvex}}
\label{sec:unlearningstronglyconvex}
\begin{proof}
 
We can analyze the linear convergence of $\theta'_t$ on $\mathcal{L}_{\mathcal{D}'}$ as follows,
\begin{align*} 
    \mathcal{L}_{\mathcal{D}'}(\theta'_{t+1}) - \mathcal{L}_{\mathcal{D}'} (\theta'_{t}) \leq & \nabla \mathcal{L}_{\mathcal{D}'}(\theta'_{t})^T (- \eta g_{\mathcal{B}'}(\theta'_t)) + \frac{L \eta^2}{2} \lVert g_{\mathcal{B}'}(\theta'_t) \rVert^2 \\
    \mathbb{E}[\mathcal{L}_{\mathcal{D}'} (\theta_{t+1}) \,|\, \mathcal{F}_t] - \mathcal{L}_{\mathcal{D}'} (\theta_{t}) \leq & \nabla \mathcal{L}_{\mathcal{D}'}(\theta_{t})^T (- \eta \nabla \mathcal{L}_{\mathcal{D}'}(\theta_t)) + \frac{L \eta^2}{2} \mathbb{E}[\lVert  g_{\mathcal{B}'}(\theta_t) \rVert^2 \,|\, \mathcal{F}_t]. \\
\overset{\text{Assumption \ref{assump:bc_assump}}}{\leq} & \nabla \mathcal{L}_{\mathcal{D}'}(\theta_{t})^T (- \eta \nabla \mathcal{L}_{\mathcal{D}'}(\theta_t)) + \frac{L \eta^2}{2} ( B \lVert \nabla \mathcal{L}_{\mathcal{D}'}(\theta) \rVert^2 + C) \\
    = & - \eta \lVert \nabla \mathcal{L}_{\mathcal{D}'}(\theta_{t}) \rVert^2 + \frac{L \eta^2 B}{2}   \lVert \nabla \mathcal{L}_{\mathcal{D}'}(\theta) \rVert^2 + \frac{L \eta^2 C}{2}.
\end{align*}
Let $\eta\leq \frac{1}{B L} $, then we have
\begin{align*} 
    \mathbb{E}[\mathcal{L}_{\mathcal{D}'} (\theta'_{t+1}) \,|\, \mathcal{F}_t] - \mathcal{L}_{\mathcal{D}'} (\theta'_{t}) \leq &  - \frac{\eta}{2} \lVert \nabla \mathcal{L}_{\mathcal{D}'}(\theta'_{t}) \rVert^2 + \frac{L \eta^2 C}{2},\\
    \overset{\text{Definition \ref{def:PL}}}{\leq} &  - \eta \mu (\mathcal{L}_{\mathcal{D}'}(\theta'_t) - \mathcal{L}_{\mathcal{D}'}^*) + \frac{L \eta^2 C}{2}, \\
    \mathbb{E}[\mathcal{L}_{\mathcal{D}'} (\theta'_{t})] - \mathcal{L}_{\mathcal{D}'}^* \leq & (1 - \eta \mu)^{t} (\mathcal{L}_{\mathcal{D}'}(\theta'_0) - \mathcal{L}_{\mathcal{D}'}^*) + \frac{L \eta^2 C}{2} \sum_{i = 0}^{t-1} (1 - \eta \mu )^i, \\
    \mathbb{E}[\mathcal{L}_{\mathcal{D}'} (\theta'_{t})] - \mathcal{L}_{\mathcal{D}'}^* \leq & (1 - \eta \mu)^{t} (\mathcal{L}_{\mathcal{D}'}(\theta'_0) - \mathcal{L}_{\mathcal{D}'}^*) + \frac{L \eta C}{2 \mu}, \\
    \leq &  (1 - \frac{\eta \mu}{2})^{t} (\mathcal{L}_{\mathcal{D}'}(\theta'_0) - \mathcal{L}_{\mathcal{D}'}^*) + \frac{L \eta C}{2 \mu}, \\
    \mathbb{E}[\mathcal{L}_{\mathcal{D}'} (\theta'_{T})] - \mathcal{L}_{\mathcal{D}'}^* \overset{\text{Assumption 
    \ref{assump:boundedell}}}{\leq} & (1 - \frac{\eta \mu}{2})^{T} 
    \ell_{\theta_0} + \frac{L \eta C}{2 \mu}. 
    \end{align*}
   Now we analyze the linear convergence of $\theta''_t$ on $\mathcal{L}_{\mathcal{D}'}$, leading to a similar result depending on $\mathbb{E}[\mathcal{L}_{\mathcal{D}'}(\theta''_0)] - \mathcal{L}_{\mathcal{D}'}^*$. This allows us to plug in the results from Lemma \ref{lemma:descent}. We have
    \begin{align*}
    \mathbb{E}[\mathcal{L}_{\mathcal{D}'} (\theta''_{K})] - \mathcal{L}_{\mathcal{D}'}^* \leq & (1 - \eta \mu)^{K} (\mathbb{E}[\mathcal{L}_{\mathcal{D}'}(\theta''_0)] - \mathcal{L}_{\mathcal{D}'}^*) + \frac{L \eta C}{2\mu},\\
    = & (1 - \eta \mu)^{K} (\mathbb{E}[\mathcal{L}_{\mathcal{D}'}(\theta_T)] - \mathcal{L}_{\mathcal{D}'}^*) + \frac{L \eta C}{2 \mu},\\
    \overset{\text{Lemma \ref{lemma:descent}}}{\leq} & (1 - \eta \mu)^{K} ((1 - \frac{\eta \mu}{2} )^T(\mathcal{L}_{\mathcal{D}'}(\theta_0)- \mathcal{L}_{\mathcal{D}'}^*) + \frac{C}{4 B \mu}+ \frac{L \eta C}{\mu}) + \frac{L \eta C}{2 \mu}, \\
    \overset{\eta \leq \frac{1}{BL}}{\leq} & (1 - \frac{\eta \mu}{2})^{K} ((1 -\frac{\eta \mu}{2} )^T(\mathcal{L}_{\mathcal{D}'}(\theta_0)- \mathcal{L}_{\mathcal{D}'}^*) + \frac{5C}{4 B \mu}) + \frac{L \eta C}{2\mu}, \\
    \overset{\text{Assumption 
    \ref{assump:boundedell}}}{\leq} & (1 - \frac{\eta \mu}{2})^{K} ((1 -\frac{\eta \mu}{2} )^T \ell_{\theta_0} + \frac{5C}{4 B \mu}) + \frac{L \eta C}{2\mu} .
\end{align*}
Let
$$T = K + \frac{\log(\ell_{\theta_0}) - \log(\frac{5C}{4 B \mu})}{\log(\frac{1}{1 - \eta \mu/2})},$$
then we have 
\begin{align*}
    \mathbb{E}[\mathcal{L}_{\mathcal{D}'} (\theta''_{K})] - \mathcal{L}_{\mathcal{D}'} ^*
    &\leq (\frac{5C}{4 B \mu})((1 - \frac{\eta \mu}{2})^{2K} +(1 - \frac{\eta \mu}{2})^{K}) +  \frac{L \eta C}{2 \mu}, \\
    \mathbb{E}[\mathcal{L}_{\mathcal{D}'} (\theta'_{T})] - \mathcal{L}_{\mathcal{D}'}^*
    &\leq (\frac{5C}{4 B \mu})(1 - \frac{\eta \mu}{2})^{K} +  \frac{L \eta C}{2 \mu} .
\end{align*}
By quadratic growth (Lemma \ref{lemma:qg}), we have
\begin{align*}
    \mathbb{E}[\lVert \theta''_K - \theta'_T \rVert^2] \leq & 2\mathbb{E}[\lVert \theta''_K - \theta^{*'} \rVert^2]  + 2\mathbb{E}[\lVert \theta'_T - \theta^{*'} \rVert^2], \\
    \leq & \frac{4}{\mu}(\mathbb{E}[\mathcal{L}_{\mathcal{D}'} (\theta''_{K})] - \mathcal{L}_{\mathcal{D}'}) + \frac{4}{\mu}(\mathbb{E}[\mathcal{L}_{\mathcal{D}'} (\theta'_{T})] - \mathcal{L}_{\mathcal{D}'}), \\
    \leq & \frac{4}{\mu} ((\frac{5C}{4 B \mu}) ((1 - \frac{\eta \mu}{2})^{2K} + 2(1 - \frac{\eta \mu}{2})^{K}) +  \frac{L \eta C}{\mu}) ,
\end{align*}
where the expectation is taken with respect to the random implementations of the unlearning algorithm producing $\theta''_K$ and the learning algorithm $\theta'_T$. 
Let 
$$\Sigma^2 =  \frac{5C}{B \mu^2}((1 - \frac{\eta \mu}{2})^{2K} + 2(1 - \frac{\eta \mu}{2})^{K}) +  \frac{4 L \eta C}{\mu^2} .$$
Then by Lemma \ref{lemma:indistinguishfirstsecond}, we obtain $(\varepsilon, 2 \delta)$-certified unlearning if we add Gaussian noise with standard deviation 
$$\sigma = \frac{\Sigma}{\varepsilon} \sqrt{\frac{2 \log (1.25/\delta)}{\delta}}.$$
\end{proof}

\section{Experiments}
\label{sec:experiments}

\subsection{Implementation Details}

We follow the experimental setup for unlearning detailed in Section 4.1 and Appendix B of \cite{mu2025rewindtodeletecertifiedmachineunlearning}, including dataset preparation, model architecture, unlearning procedure, and MIA implementations. We implement PSGD by projecting iterates onto a ball of radius $R$ centered on the origin. We use the  hyperparameters listed in Table \ref{tab:hyperparameters}. Code is open-sourced at the anonymous GitHub repository \url{https://anonymous.4open.science/r/r2d2-3753/}.

All experiments were run using PyTorch 2.5.0 and CUDA 12.1, on an Intel(R) Core(TM) i7-6850K CPU (3.60GHz) with an NVIDIA GeForce GTX 1080 GPU (8 GB VRAM) or on an Intel(R) Xeon(R) Silver 4208 CPU (2.10GHz) with an  NVIDIA RTX A6000 GPU (48 GB).

\begin{table}[ht]
\centering
\caption{R2D Experiment parameters for the eICU and Lacuna-100 datasets.}
\begin{tabular}{lll}
\toprule
\textbf{Parameter}                  & \textbf{eICU and MLP} & \textbf{Lacuna-100 and ResNet-18} \\ 
\midrule
Size of training dataset $n$        &94449      &32000        \\ 
Number of users &119282      &100        \\ 
Percent data unlearned &$\sim$ 1\%      &$\sim$ 2\% \\ 
Number of model parameters $d$ &136386      &11160258         \\ 
Batch size                 & 64  & 64    \\ 
$L$                        &0.059955      &       \\ 
$G$                        &0.820322       &        \\ 
$\eta$                     &0.001     &   0.01    \\ 
Number of training epochs                     &48      &   43   \\
$R$ & 10 & 50\\
\bottomrule
\end{tabular}
\label{tab:hyperparameters}
\end{table}

\subsection{Privacy-utility-complexity tradeoff}

To demonstrate the real-world tradeoffs implied by the derived guarantees in Theorem \ref{thm:unlearningr2d}, we vary $\varepsilon$, compute the required $\sigma$, and examine the model performance and MIA success. We estimate the constants $G$ and $L$ by sampling uniformly from the parameter space with radius $R$. Figure \ref{fig:eicuPUC} demonstrates the privacy-utility-complexity tradeoff of R2D on the nonconvex loss function on the eICU dataset, where ``Rewind Percent" is computed as $\frac{K}{T} \time 100 \%$. As expected, the behavior trends in accordance with the derived relationship between $\varepsilon$, $K$, and $\sigma$.

\subsection{Baseline Comparison} We are also interested in comparing R2D with the existing Privacy Amplification By Iteration PABI) algorithm \cite{koloskova2025certified} because of its nonconvex unlearning guarantees. We do not compare against Langevin Unlearning \cite{chien2024langevin} because, as stated in their work, ``the non-convex unlearning bound... currently is not tight enough to be applied in practice due to its exponential dependence on various hyperparameters," including the projection radius.

PABI follows a D2D-like structure but involves clipping and added noise at every step. One advantage of this approach is that it leverages  amplification to stop noise injections early, as long as the privacy budget is achieved. After this point the algorithm continues with noiseless finetuning steps to improve accuracy.

We implement PABI and R2D with $\varepsilon=1e7$ and $\delta=0.2$ and the same learning rate, batch size, projection radius, and number of iterations, and we compare their unlearning performance. We use a clipping radius of $1$ and $0.01$ weight decay.

Table \ref{tab:eicu_full} displays the results of our experiments. We observe that PABI achieves a better model performance due to a lower noise requirement and the noiseless finetuning steps. However, as with D2D, these finetuning steps also improve the performance on $\mathcal{D}_{unlearn}$, which may even increase after unlearning. In contrast, the noise for R2D is larger and degrades model performance, but it also reduces the gap between the model performance on $\mathcal{D}_{unlearn}$ and the performance on the other sets. As for the membership inference attacks, R2D defends more successfully. However, it is worth noting that MIAs are a somewhat flawed metric for unlearning and their success correlates with model performance.

\begin{figure}[h]
    \centering
    \includegraphics[width=0.99\linewidth]{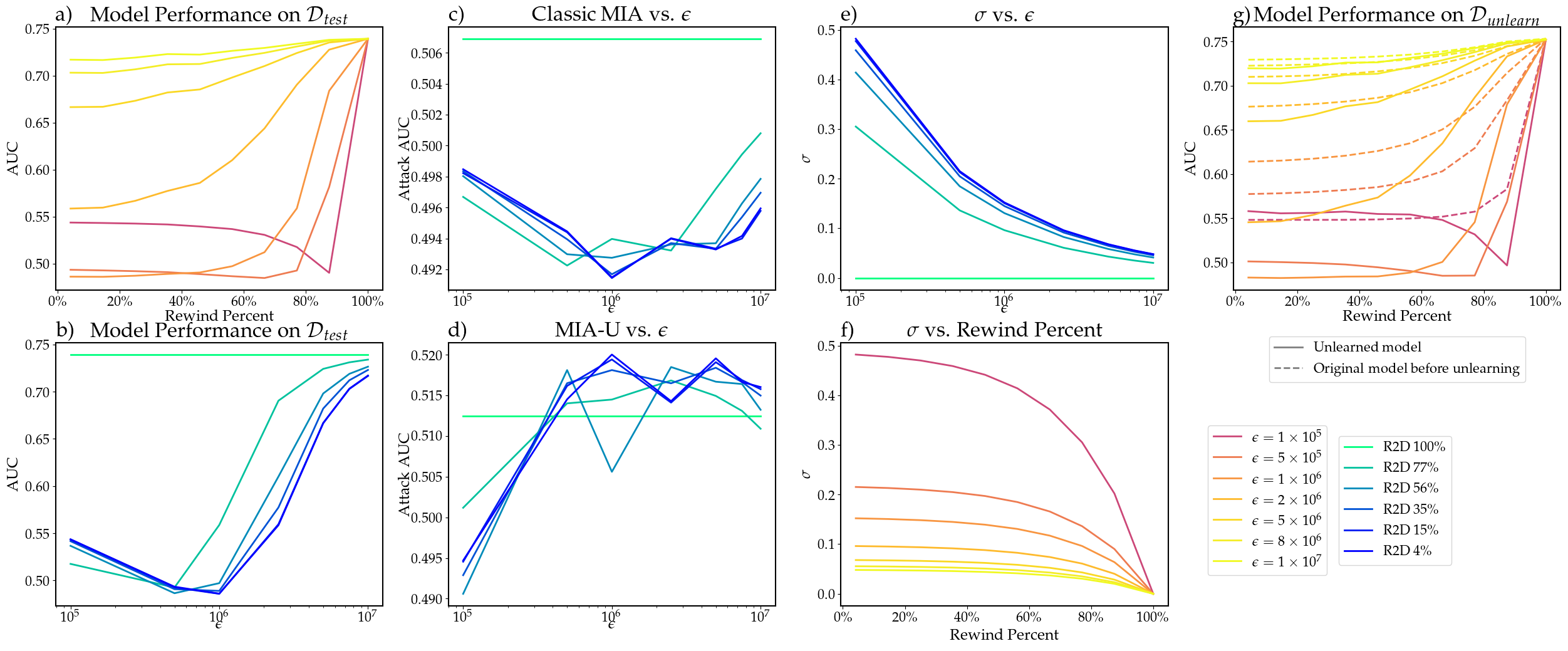}
    \caption{Privacy-utility-complexity tradeoff of PSGD R2D $(\varepsilon, \delta)$-unlearning on the eICU dataset. Rewind Percent is computed as $\frac{K}{T} \times 100\%$, where $K$ is the number of unlearning iterations and $T$ is the total number of training iterations.}
    \label{fig:eicuPUC}
\end{figure}

\begin{table}[]
    \small
    \centering
    \caption{Comparison of Certified Unlearning Algorithms on the eICU Dataset}
\begin{tabular}{lcccccc}
\toprule
 & \multicolumn{4}{c}{Model AUC} & \multicolumn{2}{c}{MIA AUC} \\
\cmidrule(lr){2-5} \cmidrule(lr){6-7}
Algorithm & $\mathcal{D}_{ood}$ & $\mathcal{D}_{retain}$ & $\mathcal{D}_{unlearn}$ & $\mathcal{D}_{test}$ & Classic MIA & MIA-U \\
\midrule
\textbf{Original Model} \\
  & 0.728316 & 0.741235 & 0.753042 & 0.739546 & &  \\

\midrule
\textbf{Certified Unlearning Algorithms} \\
R2D 14\% & 0.706665 & 0.713928 & 0.719341 & 0.716691 & 0.495783{\tiny$\pm$0.025725} & 0.515777{\tiny$\pm$0.025309} \\
R2D 35\% & 0.713061 & 0.720680 & 0.726063 & 0.723026 & 0.496954{\tiny$\pm$0.026146} & 0.514982{\tiny$\pm$0.025385} \\
PABI (14\%) & 0.728882 & 0.741641 & 0.753509 & 0.739931 & 0.507510{\tiny$\pm$0.028092} & 0.511495{\tiny$\pm$0.024315} \\
PABI (35\%) & 0.728660	 & 0.741511 & 0.752976	 & 0.739329 & 0.506902 {\tiny$\pm$0.028156} & 0.5125	{\tiny$\pm$0.025596	} \\

\midrule 
\textbf{Noiseless Retrain} (R2D 100\%) & 0.728034 & 0.741064 & 0.752053 & 0.7321 & 0.5060{\tiny$\pm$0.0285} & 0.5064{\tiny$\pm$0.0241} \\
\midrule

\bottomrule
\end{tabular}
\label{tab:eicu_full}
\end{table}

\subsection{Hyperparameters $R$ and $b$}

We conduct additional experiments to ensure appropriate choices of projection radius $R$ and batch size $b$. For $R$, following the precedent in \cite{zhang2024towards}, we implement projected SGD on a ball of radius $R$ to maintain the strictness of our results while choosing $R$ large enough to minimize impact on model utility. We assess this by examining the training loss and error for varying choices of $R$. For these experiments, we train the model (with batch size $b = 64$) and perform model selection of the parameters with lowest validation loss. As shown in Table \ref{tab:radius}, the choices of $R = 10$ and $R = 50$ for eICU and Lacuna-100 are have minimal impact on model performance.

For the batch size $b$, we desire a batch size that is small enough to highlight the effects of stochasticity but large enough so that training is stable enough to yield good model performance. We consider the training and test loss for varying choices of $b$, where for each the model is trained with the same number of \textit{iterations}. Tables \ref{tab:batchsize1} and \ref{tab:batchsize2} demonstrate the impact of varying batch sizes on the performance on the validation and training sets.

\begin{table}[]
\centering
    \caption{Choice of projection radius $R$ vs. model performance for eICU (left) and Lacuna-100 (right).}

\begin{tabular}{lcccc}
\toprule
\multicolumn{5}{c}{eICU Dataset} \\
\midrule
$R$ & Train Error & Train Loss & $G$ & $L$ \\
\midrule
1 & 0.403646 & 0.679236 & 0.722412 & 0.077918 \\
2 & 0.394975 & 0.647507 & 0.726842 & 0.076906 \\
5 & 0.313513 & 0.588481 & 0.745951 & 0.071007 \\
10 & 0.308939 & 0.581918 & 0.820322 & 0.059955 \\
15 & 0.308907 & 0.581422 & 1.084669 & 0.057535 \\
20 & 0.308907 & 0.581422 & 1.919703 & 0.082583 \\
\bottomrule
\end{tabular}
\hspace{0.5 in}
\begin{tabular}{lrr}
\toprule
\multicolumn{3}{c}{Lacuna-100 Dataset} \\
\midrule
$R$ & Train Error & Train Loss \\
\midrule
20 & 0.502000 & 0.693194 \\
30 & 0.059719 & 0.148735 \\
40 & 0.003563 & 0.013695 \\
50 & 0.020094 & 0.058069 \\
60 & 0.012969 & 0.045947 \\
70 & 0.012969 & 0.045947 \\
\bottomrule
\end{tabular}

\label{tab:radius}
\end{table}

\begin{table}[]
\centering
    \caption{Choice of batch size $b$ vs. model performance for the eICU dataset.}
\begin{tabular}{lccccc}
\toprule
 & Batch Size & Train Loss & Train Error & Validation Loss & Validation Error \\
\midrule
5 & 8 & 0.416120 & 0.250000 & 0.583045 & 0.308373 \\
4 & 16 & 0.513511 & 0.250000 & 0.582881 & 0.308373 \\
3 & 32 & 0.493724 & 0.250000 & 0.581867 & 0.307314 \\
2 & 64 & 0.551778 & 0.296875 & 0.581444 & 0.306467 \\
1 & 128 & 0.557833 & 0.265625 & 0.581424 & 0.306425 \\
0 & 256 & 0.588656 & 0.313111 & 0.588421 & 0.309813 \\
\bottomrule
\end{tabular}

\label{tab:batchsize1}
\end{table}

\begin{table}[H]
\centering
    \caption{Choice of batch size $b$ vs. model performance for the Lacuna-100 dataset.}
\begin{tabular}{lccccc}
\toprule
 & Batch Size & Train Loss & Train Error & Validation Loss & Validation Error \\
\midrule
5 & 8 & 0.315790 & 0.250000 & 0.244619 & 0.096125 \\
4 & 16 & 0.098530 & 0.000000 & 0.184725 & 0.074375 \\
3 & 32 & 0.020336 & 0.000000 & 0.201231 & 0.056375 \\
2 & 64 & 0.010567 & 0.000000 & 0.220710 & 0.049125 \\
1 & 128 & 0.000255 & 0.000000 & 0.342793 & 0.051500 \\
0 & 256 & 0.126300 & 0.048500 & 0.230133 & 0.093125 \\
\bottomrule
\end{tabular}

\label{tab:batchsize2}
\end{table}

\section{Potential Future Directions}

There are several potential future directions of this work. First, we note that other certified unlearning algorithms, such as \cite{chien2024stochastic} and \cite{koloskova2025certified}, follow a D2D-like framework with clipped or noisy gradient updates, instead of vanilla SGD. Combining their approaches with R2D instead could yield improved guarantees, especially on nonconvex functions. Second, our results for PSGD rely on  sampling \textit{with replacement} to achieve tighter bounds. Whenever a sample from $Z$ is selected during the generation of $\theta_t$, we can replace it with an i.i.d. sample from $\mathcal{D}'$ for producing $\theta'_t$, maintaining a ``valid" run of $\theta'_t$. However, if sampling occurs without replacement, then drawing one sample affects the probability of the next. Developing the corresponding theory for sampling without replacement, which is more commonly implemented in practice, would be a useful future direction.

\section{Additional Discussion of Related Work}
\label{section:morerelatedwork}

\paragraph{Differential Privacy.} The concept of certified unlearning is directly built on the existing mathematical framework of differential privacy (DP) \cite{dwork2006}. DP formalizes privacy guarantees by leveraging noise injection to limit the impact of the inclusion or exclusion of any one data sample on the output of the algorithm. The level of privacy is controlled by the parameters $\varepsilon$ and $\delta$, as stated in the definition below.

\begin{definition}
    \cite{dwork2014algorithmic} A randomized algorithm $\mathcal{M}$ with domain $\mathbb{N}^{|\mathcal{X}|}$ is $(\varepsilon, \delta)$-differentially private if for all $S \subset Range(\mathcal{M})$ and for all adjacent datasets $\mathcal{D}, \mathcal{D}' \in \mathbb{N}^{|\mathcal{X}|}$
    $$\mathbb{P}[\mathcal{M}(\mathcal{D}) \in S] \leq e^{\varepsilon} \mathbb{P}[\mathcal{M}(\mathcal{D}') \in S] + \delta.$$
\end{definition}
A common DP technique is the Gaussian mechanism, which scales Gaussian noise to the \textit{deterministic} global sensitivity, a worst-case bound on how much the algorithm output shifts when the underlying dataset is changed \cite{dwork2014algorithmic}. This technique has been applied to privatizing deterministic functions such as dataset queries or full-batch gradient descent, as well as \textit{randomized} algorithms such as SGD \cite{abadiDP, wublackboxdp, zhangDPijcai2017p548}. In particular, to bound the sensitivity, the well-known DP-SGD algorithm \cite{abadiDP} clips the gradient at each step to minimize the impact of any single data sample. Alternatively, the black-box SGD algorithm \cite{wublackboxdp} requires a uniformly bounded gradient and strongly convex loss function.



\paragraph{Biased SGD.} Our work relies on interpreting the certified unlearning problem as a form of biased or disturbed gradient descent. There are two complementary viewpoints in the literature. First, several works analyze biased SGD from a classic optimization perspective to achieve convergence to a minimum or stationary point. Typically, these works establish minimal viable assumptions on the bias or noise to achieve convergence. For example, \cite{demidovich2023guidezoobiasedsgd} summarizes biased SGD results for nonconvex and strongly convex (or PL) loss functions. In addition, \cite{pmlr-v51-hu16b} considers bandit convex optimization with absolutely bounded bias and on a bounded convex domain, and \cite{driggsbiasedSGD} considers biased SGD schemes that satisfy a specific ``memory-biased" property, like SARAH and SVRG. In our work, we utilize the groundwork laid in \cite{AjalloeianBiasedSGD} to show unlearning for D2D on strongly convex functions.

The second viewpoint of biased SGD is through the lens of nonlinear contraction theory \cite{kozachkovgeneralization} \cite{sontag2021remarks}, characterizing the bias at every step as disturbances to a noisy gradient system. The trajectory stability, which determines how far a gradient trajectory will diverge when disturbed, can be used to bound the distance between SGD training on two different loss functions. We leverage well-known properties of these systems to analyze stochastic R2D unlearning.


\end{document}